\documentclass[10pt,journal,compsoc]{IEEEtran}

\usepackage{times,amsmath,epsfig}
\usepackage{amssymb}

\usepackage{algorithm}
\usepackage{algorithmic}
\usepackage{graphicx}
\usepackage{epstopdf}
\usepackage{subfigure}

\usepackage{amsthm}

\ifCLASSOPTIONcompsoc
	\usepackage[nocompress]{cite}
\else
    \usepackage{cite}
\fi
\ifCLASSINFOpdf
\else
\fi

\begin{document}

\title{Balanced $k$-Means and Min-Cut Clustering}

\author{Xiaojun~Chang,
        Feiping~Nie,
        Zhigang~Ma,
        and~Yi~Yang
\IEEEcompsocitemizethanks{\IEEEcompsocthanksitem Xiaojun Chang and Yi Yang are with School of Information Technology and Electrical Engineering, The University of Queensland, Australia. (E-mail: cxj273@gmail.com, yi.yang@uq.edu.au)\protect\\

\IEEEcompsocthanksitem Feiping Nie is with Department of Computer Science and Engineering, University of Texas, US.

\IEEEcompsocthanksitem Zhigang Ma is with School of Computer Science, Carnegie Mellon University, US.}
}

\IEEEtitleabstractindextext{%

\begin{abstract}
Clustering is an effective technique in data mining to generate groups that are the matter of interest.
Among various clustering approaches, the family of $k$-means algorithms and min-cut algorithms gain most popularity due to their simplicity and efficacy. The classical $k$-means algorithm partitions a number of data points into several subsets by iteratively updating the clustering centers and the associated data points. By contrast, a weighted undirected graph is constructed in min-cut algorithms which partition the vertices of the graph into two sets. However, existing clustering algorithms tend to cluster minority of data points into a subset, which shall be avoided when the target dataset is balanced. To achieve more accurate clustering for balanced dataset, we propose to leverage exclusive lasso on $k$-means and min-cut to regulate the balance degree of the clustering results. By optimizing our objective functions that build atop the exclusive lasso, we can make the clustering result as much balanced as possible. Extensive experiments on several large-scale datasets validate the advantage of the proposed algorithms compared to the state-of-the-art clustering algorithms.
\end{abstract}

\begin{IEEEkeywords}
Balanced $k$-Means, Min-Cut Clustering
\end{IEEEkeywords}}

\maketitle

\IEEEdisplaynontitleabstractindextext
\IEEEpeerreviewmaketitle
\IEEEraisesectionheading{\section{Introduction}\label{sec:introduction}}

\IEEEPARstart{C}{lustering} is a fundamental research topic in data mining and is widely used for many applications in the field of artificial intelligence, statistics and social sciences \cite{clustering}\cite{algclu}\cite{clustfs} \cite{DBLP:journals/tsmc/WuHMZL12} \cite{coupledclustering} \cite{adaclu}. The objective of clustering is to partition the original data points into a number of groups so that the data points within the same cluster are close to each other while those in different clusters are far away from each other \cite{clualgorithm} \cite{DBLP:journals/tsmc/NieXL12} \cite{DBLP:conf/aaai/ChangNYH15} \cite{survey}.

Among various approaches for clustering, $K$-means and min-cut are two most popular choices in reality because of their simplicity and effectiveness \cite{kmeans}. The general procedure of traditional $K$-means (TKM) is to randomly initialize $c$ clustering centers, assign each data point to its nearest cluster and compute a new clustering center iteratively. Some researchers claim that the curse of dimensionality may deteriorate the performance of TKM \cite{dimension_reduction}. A straightforward solution of this problem is to project the original dataset to a low-dimensional subspace by dimensionality reduction, for example, PCA, before performing TKM. Discriminative analysis has been shown effective in enhancing clustering performance \cite{dimension_reduction} \cite{discluster} \cite{adaptive}. Motivated by this fact, discriminative $k$-means (DKM) \cite{dkmeans} is proposed to incorporate discriminative analysis and clustering into a single framework to formalize the clustering as a trace maximization problem.

By contrast, the min-cut clustering is realized by constructing a weighted undirected graph and then partitioning its vertices into two sets so that the total weight of the set of edges with endpoints in different sets is minimized \cite{mincutbinary} \cite{information-theoretic}. Among several graph clustering methods, min-cut tends to provide more balanced clusters as compared to other graph clustering criterion. As the within-cluster similarity in min-cut method is explicitly maximized, solving the min-cut clustering problem is nontrivial. The main difficulty lies in the constraint on the solution. Thus, to make this problem tractable, researchers proposed to relax the constraint.

Although $k$-means and min-cut have achieved promising performance in many applications, they have certain limit.Given that the distribution of the data points is balanced, one would expect the clustering result to reflect such balance. That being said, a clustering algorithm shall avoid partitioning a minority of data points into a cluster. Nonetheless, both K-means and min-cut, as well as some other similar clustering algorithms, do not guarantee balanced clustering result. In many real world data mining applications, the data from each cluster are about the same. For example, the male and female populations in the same age range cannot be very different. Therefore, for those data which are equally distributed, it is more reasonable to explicitly guarantee the clustering results balanced. 

Motivated by the limit of $k$-means and min-cut for handling balanced data, we propose to design a balanced clustering algorithm. Specifically, the exclusive lasso proposed by Zhou \emph{et al.} \cite{exclusivelasso} has been exploited in our approach to fulfill such purpose. The exclusive lasso was originally used for feature selection across multiple tasks. It models the scenario when variables in the same group compete with each other. With exclusive lasso, if one feature in a group is given a large weight, it tends to assign small or even zero weights to the other features in the same group.
Suppose that the exclusive lasso is applied on a bunch of data points across multiple categories. In a similar manner, we introduce a competition among different categories for the same data point. If more data points are clustered into one category, other categories would get fewer data points. The exclusive lass, thus in a sense, measures the balance degree of the clustering result. The smaller value the exclusive lasso obtains, the more balanced the clustering result is. With such insight, we formulate our clustering approach based on minimizing the exclusive lasso. In this paper, we particularly incorporate the exclusive lasso into $k$-means clustering and min-cut clustering, aiming to promote these two mainstream clustering approaches with stronger ability of attaining balanced clusters.

The major contributions of this paper can be summarized as follows:

\begin{enumerate}
  \item We leverage the exclusive lasso to introduce a competition among different categories for the same data point, thus enhancing the balance of the clustering result.

  \item The exclusive lasso is particularly tailored for $k$-means and min-cut. Thus, these two mostly used clustering approaches are able to achieve more balanced clustering result.

  \item The proposed algorithms are non-smooth and difficult to optimize. We propose a new iterative solution to solve the problems.

\end{enumerate}

The rest of this paper is organized as follows. After revisiting the related work on $k$-means, min-cut and the exclusive lasso in Section 2, we detail the proposed balanced $k$-means and min-cut algorithms in Section 3. Extensive experiments are given in Section 4 and Section 5 concludes this paper.

\section{Related Work}

In this section, we briefly review the research on $k$-means, min-cut and the exclusive lasso.
\subsection{The Classical $K$-means}

As one of the most efficient clustering algorithms, $k$-means clustering  has been widely applied to real-world applications. The centroids of clusters are utilized to characterize the data. The objective of $k$-means is to minimize the sum of the squared errors defined by:
\begin{equation}
J_k = \sum_{k=1}^K \sum_{i \in C_k} \|x_i - m_k\|^2,
\end{equation}
where $X = (x_1, \dots, x_n)$ denotes the data matrix and $m_k = \sum_{i \in C_k} x_i/n_k$ is the centroid of a cluster $C_k$ of $n_k$ data points.

 Previous work \cite{HKM} has shown that $H$-orthogonal non-negative matrix factorization (NMF) is equivalent to relaxed $k$-means clustering. Thus, $k$-means clustering can be reformulated using the clustering indicator as follows:

\begin{equation}
\begin{aligned}
& \min_{F,G} \|X - HF^T \|_F^2 \\
& s.t.~G_{ik} \in \{0, 1\}, \sum_{k=1}^K H_{ik} = 1, \forall i = 1, 2, \dots, n
\end{aligned}
\end{equation}
where $X \in \mathbb{R}^{d \times n}$ is the input data matrix with $n$ data represented by $d$-dimensional features; $F \in \mathbb{R}^{d \times K}$ is the clustering indicator matrix; $H \in \mathbb{R}^{n \times K}$ is the clustering assignment matrix and each row of $H$ satisfies the 1-of-$K$ coding scheme (if a data point $x_i$ is assigned to the $k$-th cluster then $H_{ik} = 1$ and $H_{ik}=0$ otherwise). In this paper, given a matrix $X = \{x_{ij}\}$, its $i$-th row, $j$-th column are denoted as $x^i$, $x_j$, respectively.

In the literature, the classical $K$-means and its variants have been applied to many data mining applications. For example, Mehrdad \emph{et al.} \cite{HKM} propose a harmony $K$-means (HKM) algorithm based on harmony search optimization method and applied it to document clustering. HKM can be proved by means of finite Markov chain theory to converge to the global optimum. Zhang \emph{et al.} \cite{initialcenter} propose a new neighborhood density method for selecting initial cluster centers for $K$-means clustering. Deepak \emph{et al.} \cite{quantizationKmeans} employ quantization schemes to retain the outcome of clustering operations. Although these methods get good performance, they have not considered how to achieve balanced clustering result when the given data points are evenly distributed. By contrast, we aim to develop a balanced $k$-means clustering algorithm that well addresses this issue.

\subsection{Min-Cut}
The principle of min-cut is rooted in graph theory. It needs a graph based on a weight matrix $W \in \mathbb{R}^{n \times n}$ built from $n$ data points $\{x_1, \dots, x_n \}$. The min-cut graph clustering objective function can be generalized as:

\begin{equation}
J = \sum_{1 \leq p < q \leq K} s(C_p, C_q) + s(C_p, C_q) = \sum_{k=1}^K s(C_k, \overline{C_k})
\end{equation}
where $K$ is the number of clusters, $C_k$ is the $k$-th cluster (sub-graph in graph G), $\overline{C_k}$ is the complement of a subset $C_k$ in graph $G$, and for any set $A$ and $B$
\begin{equation}
s(A, B) = \sum_{i \in A} \sum_{j \in B} W_{ij},
\end{equation}
\begin{equation}
d_i = \sum_j W_{ij}.
\end{equation}

We denote $q_k~(k=1, \dots, K)$ as the cluster indicators where the $i$-th element of $q_k$ is set to 1 if the $i$-th data point $x_i$ belongs to the $k$-th cluster, and 0 otherwise. For example, if the data points within each cluster are adjacent,
\begin{equation}
q_k = ( 0, \dots, 0, \overbrace{1, \dots, 1}^{n_k}, 0, \dots, 0)^T.
\end{equation}

After simple mathematical deduction, we can find that
\begin{equation}\nonumber
s(C_k, \overline{C_k} )= \sum_{i \in C_k} \sum_{j \in \overline{C_k}} W_{ij} = q_k^T( D - W) q_k
\end{equation}

\begin{equation}\nonumber
\sum_{i \in C_k} d_i = q_k^TDq_k
\end{equation}

\begin{equation}
s(C_k, C_k) = q_k^TWq_k,
\end{equation}
where $D$ is a diagonal matrix with the $i$-th diagonal element as $d_i$. The objective function of min-cut method can therefore be reformulated as:

\begin{equation}
J = \sum_{k = 1}^K q_k^T (D - W)q_k
\end{equation}

Min-Cut clustering has been applied in various applications. Wang \emph{et al.} \cite{FCSC} propose a flexible and generalized framework for constrained spectral clustering, interpret the algorithm as finding the normalized min-cut of a labeled graph, and apply it to constrained image segmentation. Dynamic graph clustering algorithm, proposed by \cite{increminimutcut} can provide strong theoretical quality guarantee on clusters. However, none of the existing work on min-cut is capable of balanced clustering when necessary, which shall be addressed by our newly proposed balanced min-cut algorithm.

\subsection{Exclusive Lasso}

Zhou \emph{et al.} propose the exclusive lasso to model the scenario when variables in the same group compete with each other. They apply it to multi-task feature selection and obtain good performance. The exclusive lasso \cite{exclusivelasso} is defined as follows:

\begin{equation}
\|\beta \|_e = \sqrt{\sum_{j=1}^d (\sum_{k=1}^m \|\beta _k^j\|)^2},
\label{lasso}
\end{equation}
where $\|\beta \|_e$ is a regularizer that controls the complexity of the combination weights.

In \cite{exclusivelasso}, the regularizer introduces an $l_1$-norm to combine the weights for the same category used by different data points and an $l_2$-norm to combine the weights of different categories. Since $l_1$-norm tends to achieve a sparse solution, the construction in the exclusive lasso essentially introduces a competition among different categories for the same data points.

In our work, the exclusive lasso is used as a balance constraint. We will prove that the value of exclusive lasso indicates the balance degree of our clustering algorithms.

\section{The proposed algorithm}

In this section, we illustrate the proposed approach in details.
\subsection{Balance Constraint}

Given $F$ as a cluster indicator matrix, the exclusive lasso of $F$ is written as 
\begin{equation}
\|F\|_e = \sqrt{\sum_{j=1}^c (\sum_{i=1}^n \|f_{ij}\|)^2}.
\end{equation}
With simple mathematical deduction, the exclusive lasso can be rewritten as: 
\begin{equation}
\|F\|_e = Tr(F^T\mathbf{1}\mathbf{1}^TF).
\end{equation}
From this equation, we can observe that the value of exclusive lasso equals the square-sum of the number of data points in each class. In the following, we prove that the most balanced clustering can be achieved by minimizing the exclusive lasso.

\newtheorem{theorem}{Theorem}
\begin{theorem}
Given $n_1 + n_2 + \dots + n_k = N$ and $n_i|_{i=1}^k \geq 0$, $\sum_{i=1}^k n_i^2$ arrives at its minimum when $n_i = \frac{N}{k}$.
\end{theorem}

\begin{proof}
According to the Cauchy inequality,
\begin{equation}\nonumber
\begin{aligned}
& (n_1^2 + n_2^2 + \dots + n_k^2)(b_1^2 + b_2^2 + \dots +  b_k^2) \\
 \geq & (a_1b_1 + a_2b_2 + \dots + a_kb_k)^2
\end{aligned}
\end{equation}
Let $b_i|_{i=1}^k = 1$, the equality holds when $n_1 = n_2 = \dots = n_k$. Hence, we can easily have the conclusion that when $n_i = \frac{N}{k}$, $\sum_{i=1}^k n_i^2$ get its minimal value.
\end{proof}

According to the above theorem, by minimizing the exclusive lasso, each cluster will have $\frac{n}{c}$ data points. The most balanced clustering result is thus obtained. Hence, we use the the exclusive lasso as the balance constraint.

\subsection{Balanced $k$-Means}
In the setting of clustering, given $n$ data points $\{x_i\}|_{i=1}^n$, we have a data matrix $X = (x_1, \dots, x_n) \in \mathbb{R}^{d \times n}$. Our goal in balanced $k$-means is to partition $\{x_i\}_|{i=1}^n$ into $K$ balanced clusters among different categories.

Noting that the exclusive lasso is capable of introducing competition among different categories, we apply the exclusive lasso to the classical $k$-means to obtain balanced clusters. The proposed objective function of balanced $k$-means is formulated as follows:

\begin{equation}
\min_{F \in Ind} \|X - HF^T\|_F^2 + \gamma \|F\|_e
\end{equation}

By substituting $\|F\|_e$ with \eqref{lasso}, the objective function can be rewritten as follows:

\begin{equation}
\min_{F \in Ind} \|X - HF^T\|_F^2 + \gamma Tr(F^T \mathbf{1} \mathbf{1}^T F)
\label{objkmeans}
\end{equation}
where $F \in Ind$ means $F\in \mathbb{R}^{n \times K}$ is an indicator matrix used for clustering; $H\in \mathbb{R}^{d \times K}$ is the clustering assignment matrix; $\gamma$ is a parameter.

The optimal $H$ and $F$ would minimize the objective function value. Since it is difficult to compute the optimal $H$ and $F$ simultaneously, we present an iterative approach to optimize this algorithm. To be more specific, we can obtain the optimal $H$ by fixing $F$ by a simple linear equation. Similarly, we can get the optimal $F$ by fixing $H$.

For a fixed $F$, by setting the derivative of \eqref{objkmeans} w.r.t $H$ to zero, we obtain

\begin{equation}
H = XF(F^TF)^{-1}
\end{equation}

Then we fix $H$, we update $F$ as follows: we update one row of $F$ each time while fixing the other rows of the prediction matrix $F$. Specifically, the updating of one row is realized by finding the element being 1 that results in the minimum of \eqref{objkmeans}. We iterate the updating of each row until convergence as shown in Algorithm 1.

\begin{algorithm}
\label{algkmeans}
\caption{Algorithm to solve the objective function of balanced $k$-means}
$\mathbf{Input}$: Data matrix $X\in \mathbb{R}^{d \times n}$ \\
$\mathbf{Output}$: Indicator matrix $F\in \mathbb{R}^{n \times K}$ \\
\begin{algorithmic}[1]
\STATE Initialize the indicator matrix $F$ randomly.
\REPEAT
\STATE Fixing $F$, compute $H$ according to $H = XF(F^TF)^{-1}$
\STATE Fixing $H$, update $F$ as follows: \\
~~Update each row of $F$ while fixing the remaining rows.
\UNTIL{CONVERGENCE}
\end{algorithmic}
$\mathbf{Return}$: Indicator matrix $F$.
\end{algorithm}

$\mathbf{Computational~Analysis}$: The computation complexity of Algorithm 1 is $O(K)$. Since the indicator matrix $F$ is sparse, this inverse operation is very efficient. When sufficient computational resources are available and parallel computing is implemented, this algorithm can be solved with desired efficiency.

$\mathbf{Convergence~Analysis}$: The following theorem guarantees the convergence of Algorithm 1.

\begin{theorem}
Algorithm 1 decreases the objective value of Eq. \eqref{objkmeans} in each iteration.
\end{theorem}
\begin{proof}
In each iteration $t$ of Algorithm 1, according to Step 3, we know that:
\begin{equation}
H_{t+1} = \min_F \|X - HF_t^T\|_{F}^2 + \gamma Tr(F_t^T\mathbf{1}\mathbf{1}^TF_t)
\end{equation}

Thus, we have:

\begin{equation}
\begin{aligned}
& \|X - H_{t+1}F_{t}^T\|_F^2 + \gamma Tr(F_{t}^T \mathbf{1}\mathbf{1}^TF_{t}) \\
 \leq & \|X - H_tF_{t}^T\|_F^2 + \gamma Tr(F_{t}^T \mathbf{1}\mathbf{1}^TF_{t})
\end{aligned}
\label{H}
\end{equation}

According to step 4, we obtain:

\begin{equation}
\begin{aligned}
& \|X - H_tF_{t+1}^T\|_F^2 + \gamma Tr(F_{t+1}\mathbf{1}\mathbf{1}^TF_{t+1}) \\
 \leq & \|X - H_tF_{t}^T\|_F^2 + \gamma Tr(F_{t}\mathbf{1}\mathbf{1}^TF_{t})
\end{aligned}
\label{F}
\end{equation}

Adding Eq. \eqref{H} and Eq. \eqref{F}, we arrive at:

\begin{equation}
\begin{aligned}
& \|X - H_{t+1}F_{t+1}^T\|_F^2 + \gamma Tr(F_{t+1}\mathbf{1}\mathbf{1}^TF_{t+1}) \\
 \leq & \|X - H_tF_{t}^T\|_F^2 + \gamma Tr(F_{t}\mathbf{1}\mathbf{1}^TF_{t})
\end{aligned}
\end{equation}
which proves that the algorithm decreases the objective function value in each iteration.
\end{proof}

According to Theorem 2, we can see that the value of the objective function \eqref{objkmeans} decrease at each iteration of Algorithm 1. In addition, it is clear that \eqref{objkmeans} is lower bounded by 0. Therefore, Algorithm 1 is guaranteed to converge. 

\subsection{Balanced Min-Cut}
We similarly aim to cluster $n$ data points $X = \{ x_1, \dots, x_n\}\in \mathbb{R}^{d \times n}$ into $K$ clusters. To begin with, we use the Gaussian function to construct a weight matrix $A$. The weight $A_{ij}$ is defined as:
\begin{eqnarray}
A_{ij} =
\left\{
\begin{array}{lll}
exp(-\frac{\|x_i - x_j\|^2}{\delta ^2}), x_i~and~x_j~are~k\\nearest~neighbors.\\
0, ~otherwise
\end{array}
\right.
\label{weightA}
\end{eqnarray}

where $\delta$ is utilized to control the spread of neighbors.
Given the weight matrix $A$ and the cluster indicator matrix $F$, the objective function of min-cut graph clustering is formulated as follows:
\begin{equation}
\min_{F \in Ind} \mathbf{1}^TA\mathbf{1} - Tr(F^TAF),
\end{equation}
which is equivalent to the following objective function:
\begin{equation}
\max_{F \in Ind} Tr(F^TAF)
\end{equation}

We further incorporate the exclusive lasso into min-cut and get the following objective function:

\begin{equation}
\max_{F \in Ind} Tr(F^TAF) - \gamma \|F\|_e
\end{equation}

In the same manner, we substitute $\|F\|_e$ with \eqref{lasso} and rewrite the objective function as follows:

\begin{equation}
\max_{F \in Ind} Tr(F^TAF) - \gamma Tr(F^T\mathbf{1}\mathbf{1}^TF)
\end{equation}

With a simple mathematical deduction, the objective function is rewritten as:

\begin{equation}
\max_{F \in Ind} Tr\left(F^T (\rho I + A - \gamma \mathbf{1} \mathbf{1}^T)F\right),
\label{mincutobj}
\end{equation}
where $\rho$ is a large enough constant to make $\rho I + A - \gamma \mathbf{1}\mathbf{1}^T$ positive-definite. Defining $B=(\rho I + A - \gamma \mathbf{1}\mathbf{1}^T) F$, we update $F$ by solving $\max_{F \in Ind} Tr(F^T B)$. $F$ is iteratively updated until convergence as shown in
Algorithm 2.

\begin{algorithm}
\label{algmincut}
\caption{Algorithm to solve the objective function of balanced min-cut}
$\mathbf{Input}$: Data matrix $X$  \\
$\mathbf{Output}$: Indicator matrix $F$ \\
\begin{algorithmic}[1]
\STATE Compute the weight matrix $A$ using Eq \eqref{weightA}.
\REPEAT
\STATE Compute $B$ according to $B = (\rho I + A - \gamma \mathbf{1}\mathbf{1}^T)F$
\STATE Update $F$ by solving $\max_{F \in Ind} Tr(F^TB)$
\UNTIL{CONVERGENCE}
\end{algorithmic}
$\mathbf{Return}$: Indicator matrix $F$
\end{algorithm}

$\mathbf{Compuational~Analysis}$: The computation complexity of Algorithm 2 is $O(n)$.

$\mathbf{Convergence~Analysis}$: The following theorem guarantees the convergence of Algorithm 2.

\begin{theorem}
\label{the3}
Algorithm 2 increases the objective function value of Eq. \eqref{mincutobj} in each iteration.
\end{theorem}

\begin{proof}
In the Steps 3 and 4 of Algorithm 2, we denote the updated $B$ and $F$ by $\hat{B}$ and $\hat{F}$, respectively. Since the updated $B$ and $F$ are the optimal solutions of the problem $\max_{F \in Ind}Tr(F^TB)$, we have:
\begin{equation}
Tr(\hat{F}^T(\rho I + A - \gamma \mathbf{1}\mathbf{1}^T)F) \geq Tr(F^T(\rho I + A - \gamma \mathbf{1}\mathbf{1}^T)F),
\end{equation}
which proves that the algorithm increase the objective function value in each iteration.
\end{proof}

According to Theorem \ref{the3}, we can observe that the value of objective function \eqref{mincutobj} increases at each iteration of Algorithm 2. Therefore, Algorithm 2 is proved to converge.

\section{Experiment}

In this section, extensive experiments are conducted to evaluate the proposed clustering methods. We give two sets of experiments. The first one is to compare the proposed balanced $K$-means clustering to $K$-means based clustering algorithms, including the classical $K$-means (KM) clustering, DisCluster (DC), DisKmeans (DKM) clustering \cite{dkmeans}, AKM \cite{HKM} and HKM \cite{HKM}.
\begin{table*}[!ht]
\caption{A BRIEF SUMMARY OF THE EXPERIMENTAL DATASETS.}
\centering
\begin{tabular}{|c||r|c|c|c|}
\hline
Dataset & Size & Dimension of Features & Class Number  \\
\hline
MNIST Handwritten digit dataset & 70,000 & 1024 & 10 \\
\hline
USPS Handwritten Digit Data Set & 9298 & 256 & 10 \\
\hline
YaleB Face Data Set& 2414 & 1024 & 38  \\
\hline
ORL Face Data Set & 400 & 1024 & 40  \\
\hline
JAFFE Facial Expression Data Set& 213 & 676 & 10  \\
\hline
HumanEVA Motion Data Set & 10000 & 168 & 10  \\
\hline
Coil20 Object Data Set & 1440 & 1024 & 20  \\
\hline
CMU-PIE face dataset & 41,368 & 1024 & 68 \\
\hline
UMIST face dataset & 564 & 1024 & 20 \\
\hline
\end{tabular}
\label{setting}
\end{table*}
The second one is to compare the proposed balanced min-cut clustering to the classical min-cut clustering, MinMax Cut clustering, Ratio Cut clustering and Normalized Cut clustering algorithms.

\begin{table*}[!ht]
\renewcommand{\arraystretch}{1.5}
\caption{Performance comparison (Clustering Accuracy $\pm$ STANDARD DEVIATION) of clustering accuracy using $k$-means, DisCluster, DisKmeans, AKM, HKM and Balanced $k$-means on nine benchmark datasets. From the experimental result, we can observe that the proposed algorithm consistently outperforms the other comparison algorithms.}
\centering
\begin{tabular}{|c||c|c|c|c|c|c|c|c|}
\hline
  &  $k$-means &  DisCluster &  DisKmeans   &  AKM  &  HKM  & Balanced $k$-means  \\
\hline \hline
MNIST & $52.6 \pm 3.3$ & $53.7 \pm 2.4$ & $54.2 \pm 3.4$ & $52.2 \pm 3.3$ & $55.4 \pm 3.1$ & $\mathbf{57.3 \pm 2.4}$ \\
\hline
USPS & $65.8 \pm 2.5$ & $67.4 \pm 2.8$ & $70.4 \pm 2.6$ & $66.3 \pm 2.9$ & $71.5 \pm 2.3$ & $\mathbf{73.4 \pm 2.8}$  \\
\hline
YaleB & $16.3 \pm 1.1$ & $35.2 \pm 2.3$ & $39.7 \pm 2.5$ & $16.8 \pm 0.8$ & $41.3 \pm 3.2$ & $\mathbf{43.5 \pm 1.8}$  \\
\hline
ORL & $37.2 \pm 1.6$ & $41.2 \pm 2.1$ & $43.9 \pm 1.8$ & $37.4 \pm 1.5$ & $44.4 \pm 2.7$ & $\mathbf{47.2 \pm 2.2}$  \\
\hline
JAFFE & $58.8 \pm 2.2$ & $59.4 \pm 2.7$ & $59.9 \pm 2.5$ & $59.0 \pm 2.8$ & $60.5 \pm 1.9$ & $\mathbf{61.2 \pm 1.8}$ \\
\hline
HumanEVA & $43.2 \pm 3.2$ & $44.2 \pm 3.1$ & $45.1 \pm 2.3$ & $43.8 \pm 3.4 $ & $46.3 \pm 2.6$ & $\mathbf{47.7 \pm 2.5}$  \\
\hline
Coil20 & $68.4 \pm 2.8$ & $65.3 \pm 2.6$ & $61.3 \pm 2.3$ & $67.9 \pm 2.7$ & $70.3 \pm 2.4$ & $\mathbf{73.1 \pm 2.3}$  \\
\hline
CMU-PIE & $19.5 \pm 0.8$ & $49.8 \pm 2.7$ & $55.5 \pm 2.9$ & $21.2 \pm 1.1$ & $56.1 \pm 2.2$ & $\mathbf{57.8 \pm 2.4}$  \\
\hline
UMIST & $39.5 \pm 2.1$ & $41.3 \pm 2.6$ & $43.2 \pm 2.4$ & $39.1 \pm 1.8$ & $44.1 \pm 2.6$ & $\mathbf{46.4 \pm 2.5}$  \\
\hline
\end{tabular}
\label{kmeansaccuracy}
\end{table*}

\subsection{Datasets}

A variety of datasets are used in our experiments which are described as follows.

\begin{enumerate}
\item MNIST Handwritten Digit Dataset: The MNIST handwritten digit dataset \cite{mnist} is a large-scale dataset of handwritten digits. It is widely used as a test bed in data mining. The dataset contains 60,000 training images and 10,000 testing images. We merge all the training and testing images in the experiments. The pixel values are used as feature representation.
\item USPS handwritten digit dataset: We additionally use the USPS dataset to validate the performance on handwritten digit recognition. The dataset consists of 9298 gray-scale handwritten digit images. We resize the images to $16 \times 16$ and use pixel values as the features.
\item YaleB face dataset: The YaleB dataset \cite{yaleb} contains 2414 near frontal images from 38 persons under different illuminations. Each image is resized to $32 \times 32$ and the pixel value is used as feature representation.
\item ORL face dataset: The ORL dataset \cite{ORL} consists of 40 different subjects with 10 images each. We also resize each image to $32 \times 32$ and use pixel values to represent the images.
\item JAFFE Japanese Female Facial Expression dataset: The JAFFE dataset \cite{JAFFE} consists of 213 images of different facial expressions from 10 different Japanese female models. The images are resized to $26 \times 26$ and represented by pixel values.
\item HumanEVA Motion dataset: The HumanEVA dataset is used to evaluate the performance of our algorithm in terms of 3D motion annotation~\footnote{http://vision.cs.brown.edu/humaneva/}. This dataset contains five types of motions. Based on the 16 joint coordinates in 3D space, 1590 geometric pose descriptors are extracted using the method proposed in \cite{humaneva} to represent 3D motion data.
\item Coil20 Object dataset: We use the Coil20 dataset \cite{COIL20} for object recognition. This dataset includes 1440 gray-scale images with 20 different objects. In our experiment, we resize each image to $32 \times 32$ and use pixel values as the features.
\item CMU-PIE dataset: The CMU-PIE face dataset consists of 41,368 images of 68 people. Each person was imaged under 13 different poses, 43 different illumination conditions, and with 4 different expressions. We also use the pixel values as the feature representations.
\item UMIST face dataset: The UMIST face dataset consists of 564 images of 20 individuals with mixed race, gender and appearance. Each individual is shown in a range of poses from profile to frontal views. The pixel value is used as the feature representation.
\end{enumerate}

Table 1 gives a brief summary of all the experimental datasets.

\begin{table*}[!ht]
\renewcommand{\arraystretch}{1.5}
\caption{Performance comparison (NMI $\pm$ STANDARD DEVIATION) of clustering accuracy using $k$-means, DisCluster, DisKmeans, AKM, HKM and Balanced $k$-means on nine benchmark datasets. From the experimental result, we can observe that the proposed algorithm consistently outperforms the other comparison algorithms.}
\centering
\begin{tabular}{|c||c|c|c|c|c|c|c|c|}
\hline
  &  $k$-means &  DisCluster &  DisKmeans   &  AKM  &  HKM  & Balanced $k$-means  \\
\hline \hline
MNIST & $61.7 \pm 2.5$ & $62.5 \pm 2.8$ & $63.1 \pm 2.6$ & $61.9 \pm 2.1$ & $64.3 \pm 3.1$ & $\mathbf{66.1 \pm 2.9}$ \\
\hline
USPS & $60.8 \pm 2.3$ & $61.4 \pm 2.5$ & $61.9 \pm 2.1$ & $61.0 \pm 2.6$ & $62.5 \pm 2.2$ & $\mathbf{63.7 \pm 2.5}$  \\
\hline
YaleB & $19.5 \pm 1.8$ & $30.1 \pm 2.1$ & $31.3 \pm 2.5$ & $19.8 \pm 2.2$ & $43.8 \pm 3.2$ & $\mathbf{46.5 \pm 2.3}$  \\
\hline
ORL & $68.7 \pm 1.8$ & $69.2 \pm 2.5$ & $69.9 \pm 1.8$ & $68.9 \pm 1.7$ & $71.1 \pm 2.3$ & $\mathbf{73.2 \pm 2.4}$  \\
\hline
JAFFE & $63.2 \pm 2.5$ & $64.1 \pm 2.2$ & $64.8 \pm 2.8$ & $62.8 \pm 2.5$ & $66.2 \pm 1.9$ & $\mathbf{68.4 \pm 2.2}$ \\
\hline
HumanEVA & $75.3 \pm 2.5$ & $76.1 \pm 2.1$ & $77.3 \pm 2.4$ & $75.1 \pm 2.8 $ & $78.2 \pm 2.4$ & $\mathbf{79.5 \pm 2.1}$  \\
\hline
Coil20 & $59.3 \pm 2.3$ & $60.5 \pm 2.3$ & $61.2 \pm 2.8$ & $59.8 \pm 2.7$ & $63.2 \pm 2.9$ & $\mathbf{65.1 \pm 2.7}$  \\
\hline
CMU-PIE & $24.2 \pm 2.3$ & $25.2 \pm 2.8$ & $25.8 \pm 2.5$ & $24.7 \pm 1.6$ & $57.8 \pm 2.4$ & $\mathbf{59.3 \pm 2.6}$  \\
\hline
UMIST & $63.7 \pm 2.4$ & $64.4 \pm 2.8$ & $65.3 \pm 2.5$ & $64.1 \pm 2.1$ & $66.8 \pm 2.4$ & $\mathbf{68.1 \pm 2.3}$  \\
\hline
\end{tabular}
\label{kmeansnmi}
\end{table*}

\subsection{Parameter Setting}
There are three parameters in our algorithms. The first one is the number of nearest neighbors and the second one is the parameter $\delta$ in Eq. \eqref{weightA}. Following , we set the number of nearest neighbors to 5 in the experiments. The self-tune clustering method is utilized to determine the parameter $\delta$. For the regularization parameter $\gamma$ in Eq. \eqref{objkmeans} and Eq. \eqref{mincutobj}, we tune them by a "grid-search" strategy from $\{10^{-6}, 10^{-4}, 10^{-2}, 10^0, 10^2, 10^4, 10^6 \}$. We similarly tune the regularization parameters of all the comparison algorithms from the aforementioned range. The best results of all the comparison algorithms are reported.

\begin{table*}[!ht]
\renewcommand{\arraystretch}{1.5}
\caption{Performance comparison (Clustering Accuracy $\pm$ STANDARD DEVIATION) of clustering accuracy using the classical Min-Cut clustering, MinMax Cut clustering, Ratio Cut clustering, Normalized Cut clustering and Balanced Min-Cut clustering on nine benchmark datasets. From the experimental result, we can observe that the proposed algorithm consistently outperforms the other comparison algorithms.}
\centering
\begin{tabular}{|c||c|c|c|c|c|c|c|c|}
\hline
  &  Min-Cut &  Ratio Cut  &  Normalized Cut & MinMax Cut &  Balanced Min-Cut  \\
\hline \hline
MNIST & $56.4 \pm 2.8$ & $57.8 \pm 2.2$ & $58.4 \pm 3.2$ & $59.2 \pm 2.4$  & $\mathbf{61.4 \pm 2.0}$ \\
\hline
USPS & $72.3 \pm 2.3$ & $73.6 \pm 2.5$ & $73.9 \pm 2.2$ & $75.8 \pm 2.1$  & $\mathbf{77.5 \pm 2.7}$  \\
\hline
YaleB & $37.9 \pm 2.8$ & $38.2 \pm 2.4$ & $38.6 \pm 2.1$ & $42.2 \pm 2.6$  & $\mathbf{46.1 \pm 2.3}$  \\
\hline
ORL & $45.8 \pm 1.9$ & $46.9 \pm 2.4$ & $47.6 \pm 2.3$ & $48.8 \pm 1.7$ &  $\mathbf{50.1 \pm 2.7}$  \\
\hline
JAFFE & $60.4 \pm 2.6$ & $61.1 \pm 2.3$ & $62.5 \pm 2.8$ & $62.8 \pm 2.1$  & $\mathbf{64.3 \pm 2.4}$ \\
\hline
HumanEVA & $47.2 \pm 2.9$ & $48.3 \pm 2.5$ & $48.8 \pm 2.7$ & $49.5 \pm 3.3 $ & $\mathbf{50.9 \pm 2.8}$  \\
\hline
Coil20 & $70.3 \pm 2.4$ & $71.8 \pm 2.2$ & $77.6 \pm 2.8$ & $78.3 \pm 2.3$  & $\mathbf{81.6 \pm 1.9}$  \\
\hline
CMU-PIE & $56.2 \pm 1.3$ & $57.4 \pm 2.9$ & $58.3 \pm 2.6$ & $59.1 \pm 1.8$  & $\mathbf{61.3 \pm 2.8}$  \\
\hline
UMIST & $59.6 \pm 2.5$ & $60.1 \pm 2.2$ & $60.8 \pm 2.1$ & $62.9 \pm 1.4$ & $\mathbf{64.6 \pm 2.3}$  \\
\hline
\end{tabular}
\label{mincutaccuracy}
\end{table*}

\begin{table*}[!ht]
\renewcommand{\arraystretch}{1.5}
\caption{Performance comparison (Clustering Accuracy $\pm$ STANDARD DEVIATION) of clustering accuracy using the classical Min-Cut clustering, MinMax Cut clustering, Ratio Cut clustering, Normalized Cut clustering and Balanced Min-Cut clustering on nine benchmark datasets. From the experimental result, we can observe that the proposed algorithm consistently outperforms the other comparison algorithms.}
\centering
\begin{tabular}{|c||c|c|c|c|c|c|c|c|}
\hline
  &  Min-Cut &   Ratio Cut  &  Normalized Cut & MinMax Cut &  Balanced Min-Cut  \\
\hline \hline
MNIST & $65.3 \pm 2.9$ & $66.8 \pm 2.6$ & $67.4 \pm 3.2$ & $68.1 \pm 2.5$  & $\mathbf{69.4 \pm 2.3}$ \\
\hline
USPS & $66.5 \pm 2.3$ & $67.9 \pm 2.5$ & $68.4 \pm 2.9$ & $69.8 \pm 2.7$  & $\mathbf{71.2 \pm 2.2}$  \\
\hline
YaleB & $43.6 \pm 1.8$ & $45.2 \pm 2.6$ & $46.4 \pm 2.1$ & $47.2 \pm 1.9$  & $\mathbf{49.1 \pm 2.4}$  \\
\hline
ORL & $78.1 \pm 1.9$ & $79.5 \pm 2.6$ & $80.3 \pm 2.2$ & $80.9 \pm 1.8$ &  $\mathbf{83.2 \pm 2.6}$  \\
\hline
JAFFE & $67.8 \pm 2.5$ & $69.1 \pm 2.3$ & $69.9 \pm 2.8$ & $70.3 \pm 2.4$  & $\mathbf{73.5 \pm 1.7}$ \\
\hline
HumanEVA & $77.4 \pm 3.5$ & $78.6 \pm 2.8$ & $79.2 \pm 2.4$ & $80.4 \pm 3.1 $ & $\mathbf{82.5\pm 2.1}$  \\
\hline
Coil20 & $59.8 \pm 2.9$ & $61.4 \pm 2.3$ & $62.7 \pm 2.5$ & $63.6 \pm 2.8$  & $\mathbf{66.2 \pm 2.6}$  \\
\hline
CMU-PIE & $55.5 \pm 2.1$ & $61.4 \pm 2.6$ & $62.3 \pm 2.7$ & $62.8 \pm 2.3$  & $\mathbf{63.1 \pm 2.8}$  \\
\hline
UMIST & $82.7 \pm 2.8$ & $90.1 \pm 2.1$ & $91.2 \pm 2.7$ & $92.5 \pm 2.3$ & $\mathbf{94.8 \pm 2.9}$  \\
\hline
\end{tabular}
\label{mincutnmi}
\end{table*}

\begin{figure*}[!ht]
\centering
\subfigure[]{
\includegraphics[scale=0.2]{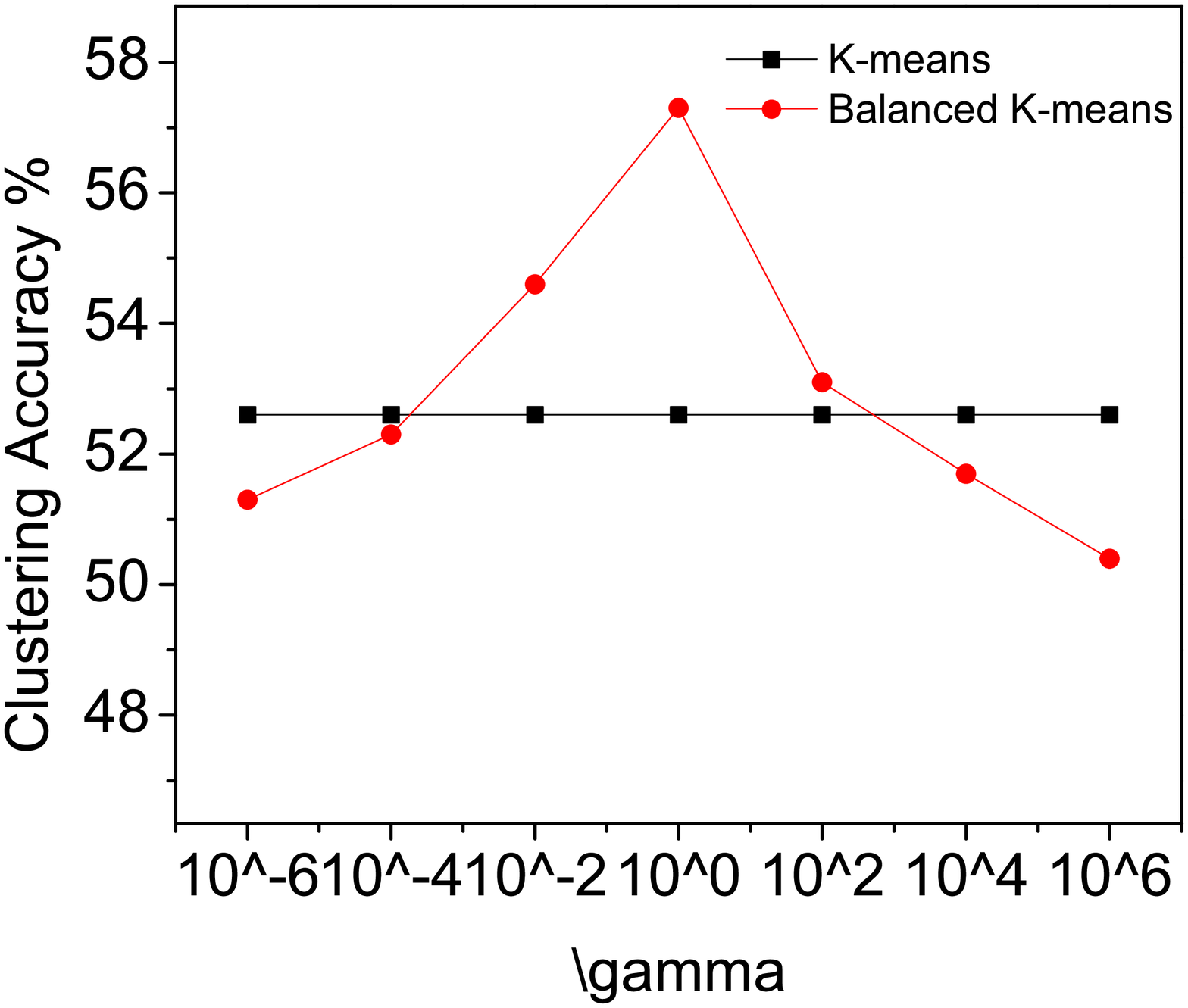}}
\subfigure[]{
\includegraphics[scale=0.2]{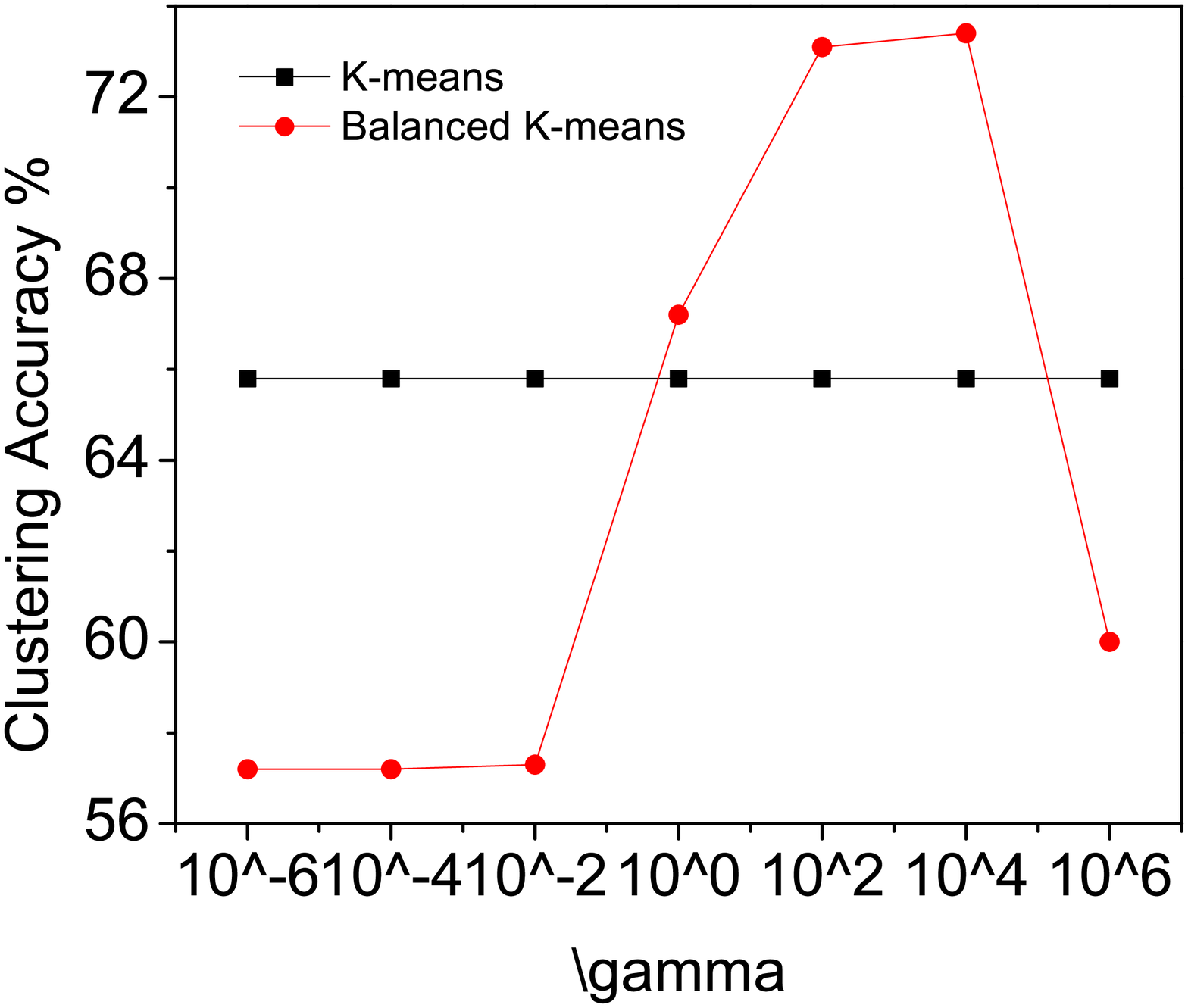}}
\subfigure[]{
\includegraphics[scale=0.2]{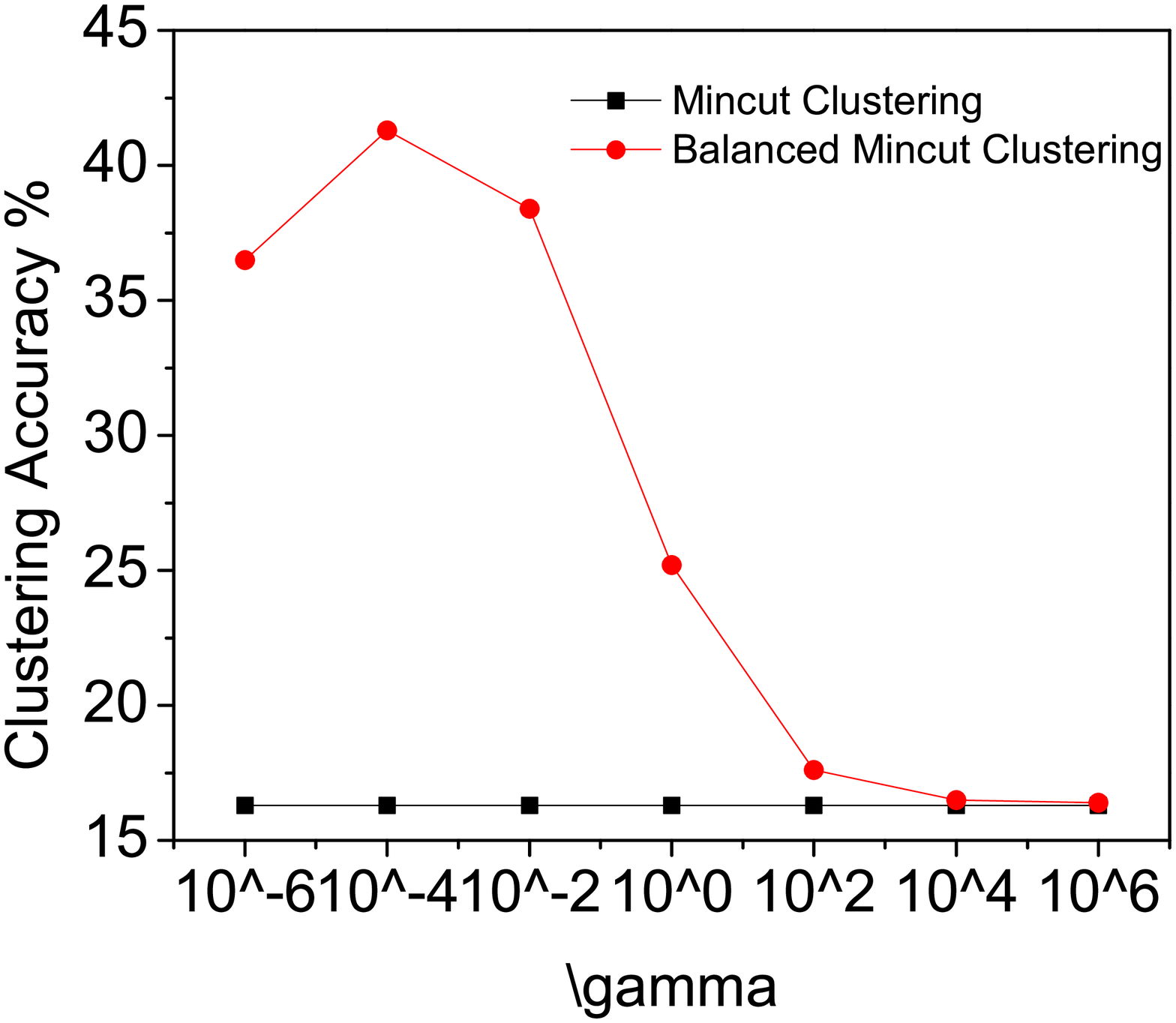}}
\subfigure[]{
\includegraphics[scale=0.2]{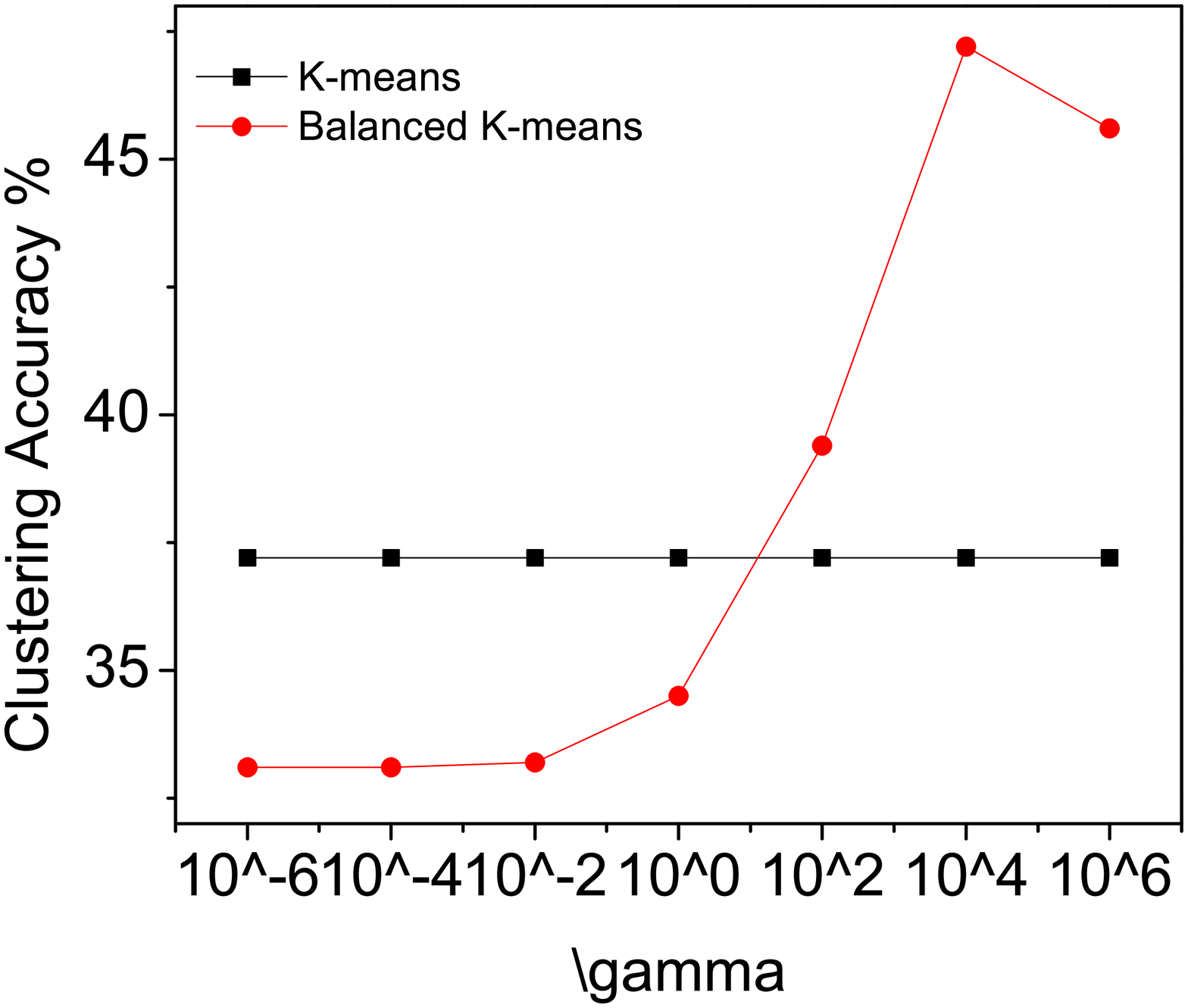}}
\subfigure[]{
\includegraphics[scale=0.2]{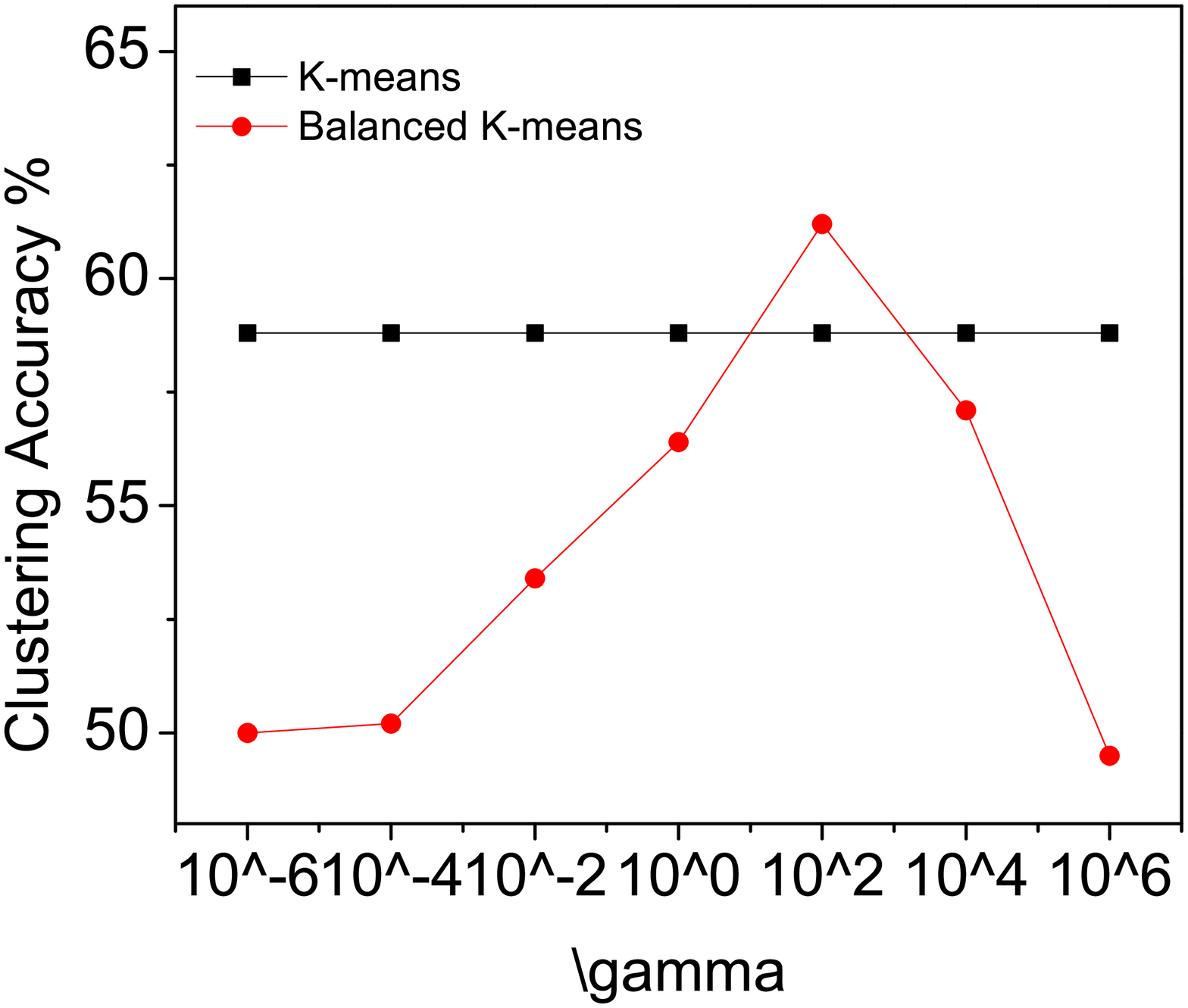}}
\subfigure[]{
\includegraphics[scale=0.2]{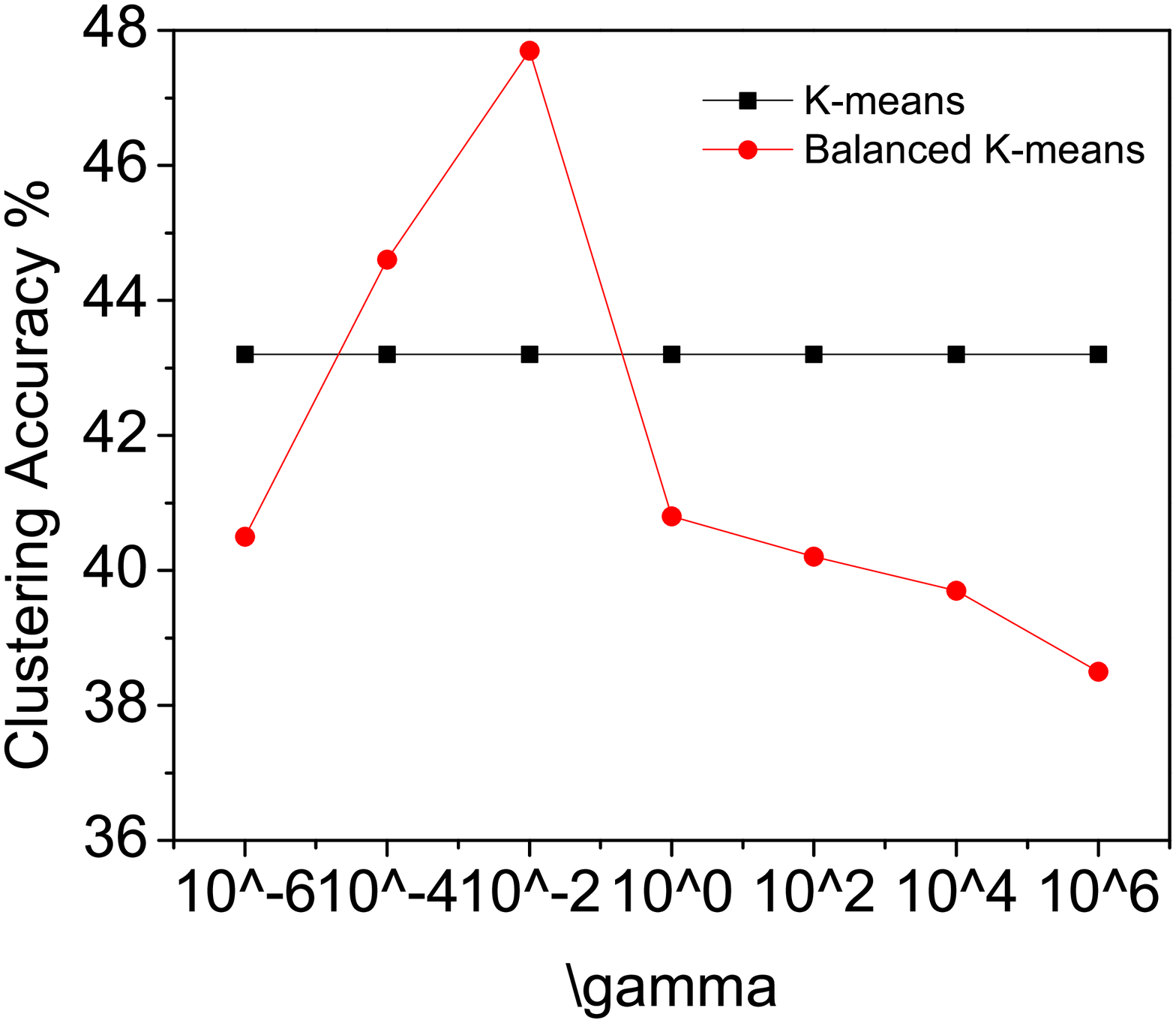}}
\subfigure[]{
\includegraphics[scale=0.2]{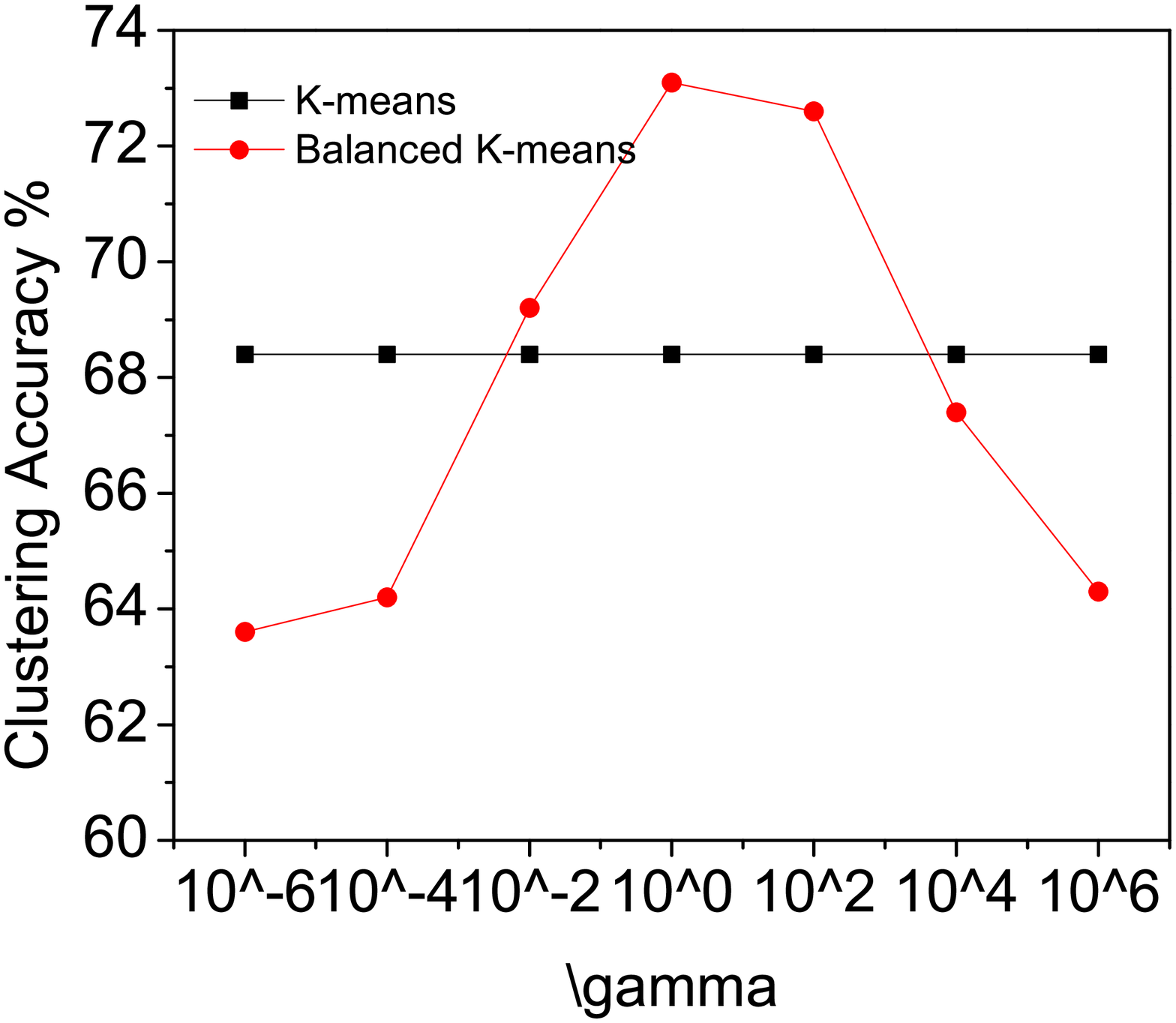}}
\subfigure[]{
\includegraphics[scale=0.2]{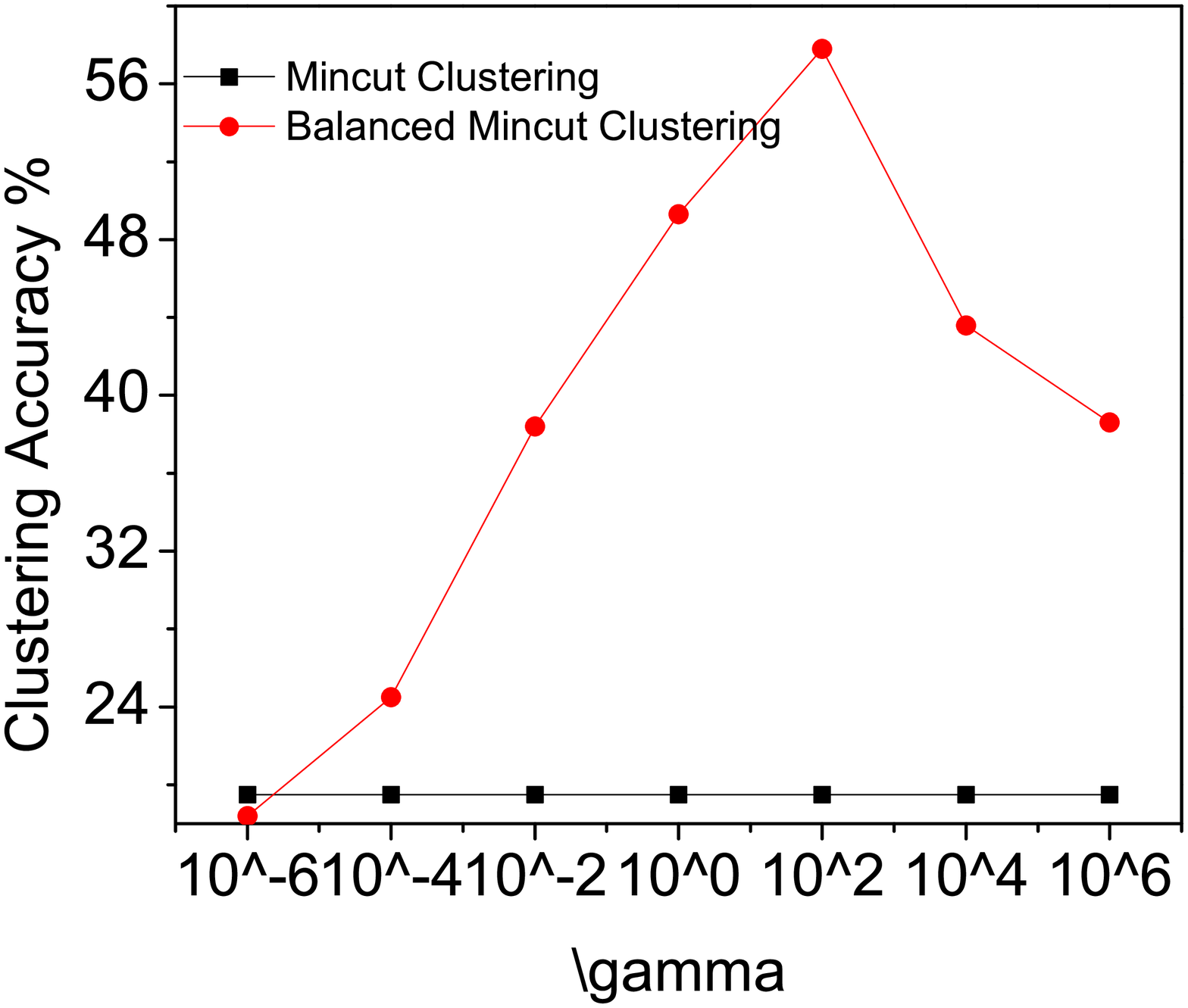}}
\subfigure[]{
\includegraphics[scale=0.2]{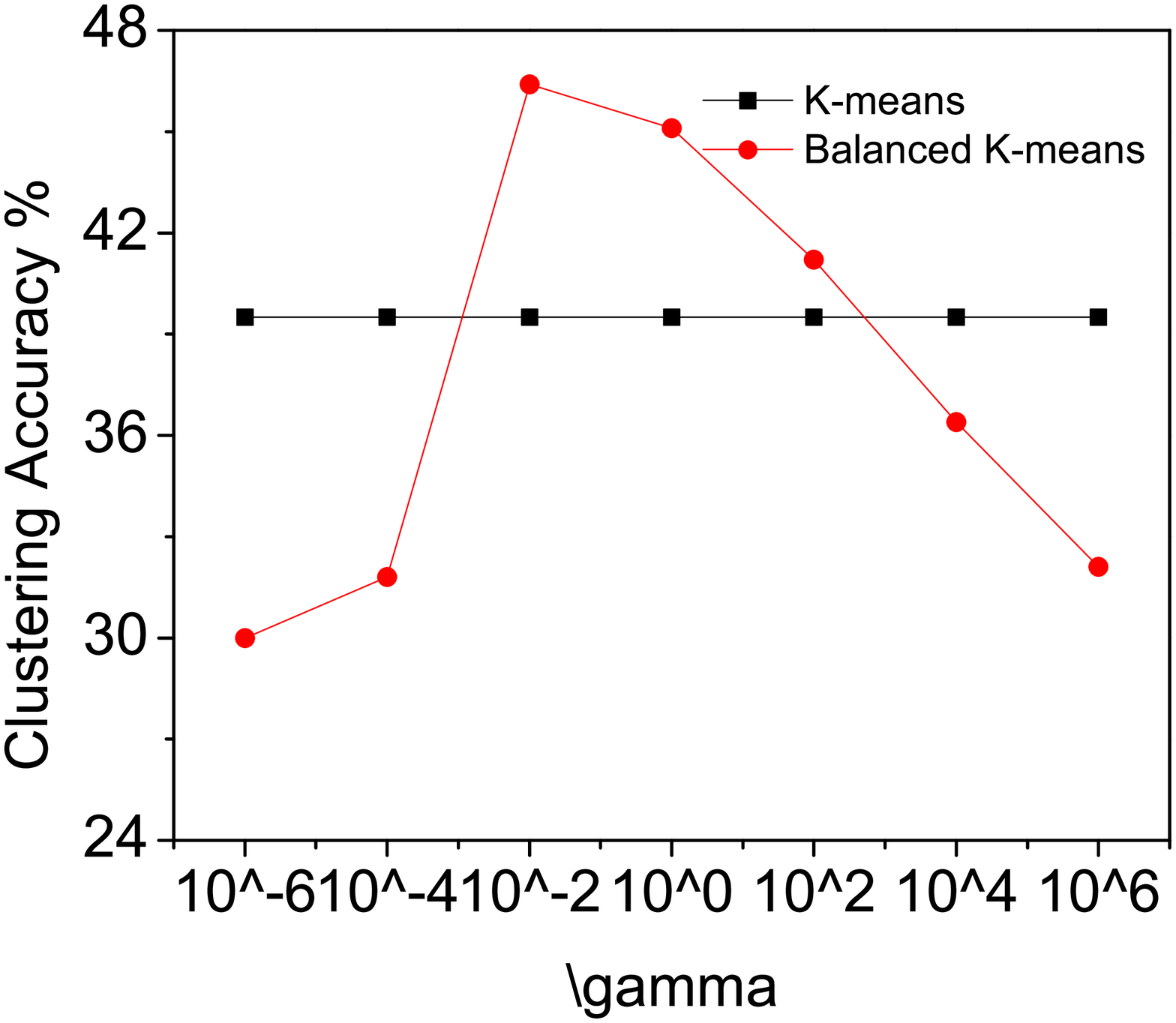}}
\caption{Parameter sensitivity of Balanced $k$-means.  (a) MNIST (b) USPS (c) YaleB (d) ORL (e) JAFFE (f) HumanEVA (g) Coil20 (h) CMU-PIE (i) UMIST. From the results, we can observe that the parameter has a significant impact on the performance.}
\label{kmeanssen}
\end{figure*}

\subsection{Evaluation Metrics}

Following related work, we adopt clustering accuracy (ACC) and normalized mutual information (NMI) as our evaluation metrics in our experiments.

Let $q_i$ represent the clustering label result from a clustering algorithm and $p_i$ represent the corresponding ground truth label of arbitrary data point $x_i$. Then $ACC$ is defined as follows:

\begin{equation}
ACC = \frac{\sum_{i=1}^n \delta (p_i, map(q_i))}{ n },
\end{equation}
where $\delta(x, y) = 1$ if $x=y$ and $\delta (x, y) = 0$ otherwise. $map(q_i)$ is the best mapping function that permutes clustering labels to match the ground truth labels using the Kuhn-Munkres algorithm. A larger ACC indicates a better clustering performance.

For any two arbitrary variable $P$ and $Q$, NMI is defined as follows \cite{NMI}:

\begin{equation}
NMI = \frac{I(P, Q)}{\sqrt{H(P)H(Q)}},
\end{equation}
where $I(P, Q)$ computes the mutual information between $P$ and $Q$, and $H(P)$ and $H(Q)$ are the entropies of $P$ and $Q$. Let $t_l$ represent the number of data in the cluster $\mathcal{C}_l(1 \leq l \leq c)$ generated by a clustering algorithm and $\widetilde{t_h}$ represent the number of data points from the $h$-th ground truth class. NMI metric is then computed as follows \cite{NMI}:

\begin{equation}
NMI = \frac{\sum_{l=1}^c \sum_{h=1}^c t_{l,h} log(\frac{n \times t_{l, h}}{•t_l\widetilde{t_h}})}{\sqrt{(\sum_{l=1}^c t_l \log \frac{t_l}{n})(\sum_{h=1}^c \widetilde{t_h} \log \frac{\widetilde{t_h}}{n})}},
\end{equation}
where $t_{l,h}$ is the number of data samples that lies in the intersection between $\mathcal{C}_l$ and $h$th ground truth class. Similarly, a larger NMI indicates a better clustering performance.

\subsection{Comparison among $k$-means based methods}

In this section, we report the performance comparison using $k$-means, DisCluster, DisKmeans, AKM, HKM and Balanced $k$-means in terms of clustering accuracy (ACC) and NMI in Table \ref{kmeansaccuracy} and Table \ref{kmeansnmi}.

From the experimental results, we have the following observations:

\begin{enumerate}
\item When compared to the classical $k$-means clustering, DisCluster and DisKmeans algorithms, DisCluster and DisKmeans generally have better performance. This may be because discriminative dimension reduction is integrated into a single framework. Thus, each cluster is more identifiable, which helps enhance the clustering performance. We can therefore conclude that discriminative information is beneficial for clustering.
\item HKM achieves the second best performance among the comparison algorithms, which indicates that most active points changing their cluster assignments at each iteration are located on or near the cluster boundaries.
\item The proposed balanced $k$-means always gets the best performance on all the datasets. This experimental result demonstrates that the exclusive lasso is able to pose balance constraint to $k$-means clustering. By minimizing the exclusive lasso, the most balanced clustering result is obtained.
\end{enumerate}

\begin{figure*}[!ht]
\centering
\subfigure[]{
\includegraphics[scale=0.21]{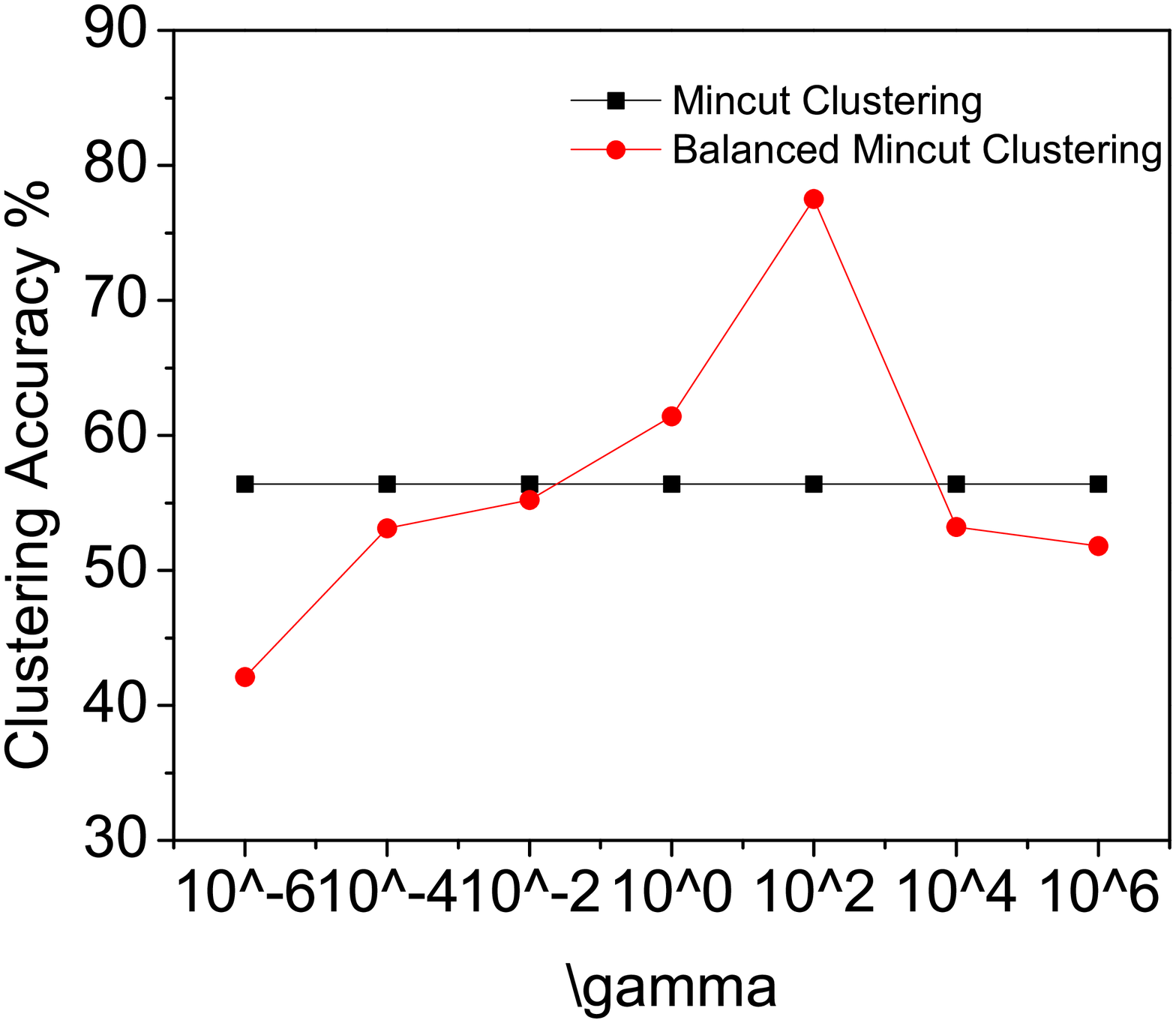}}
\subfigure[]{
\includegraphics[scale=0.21]{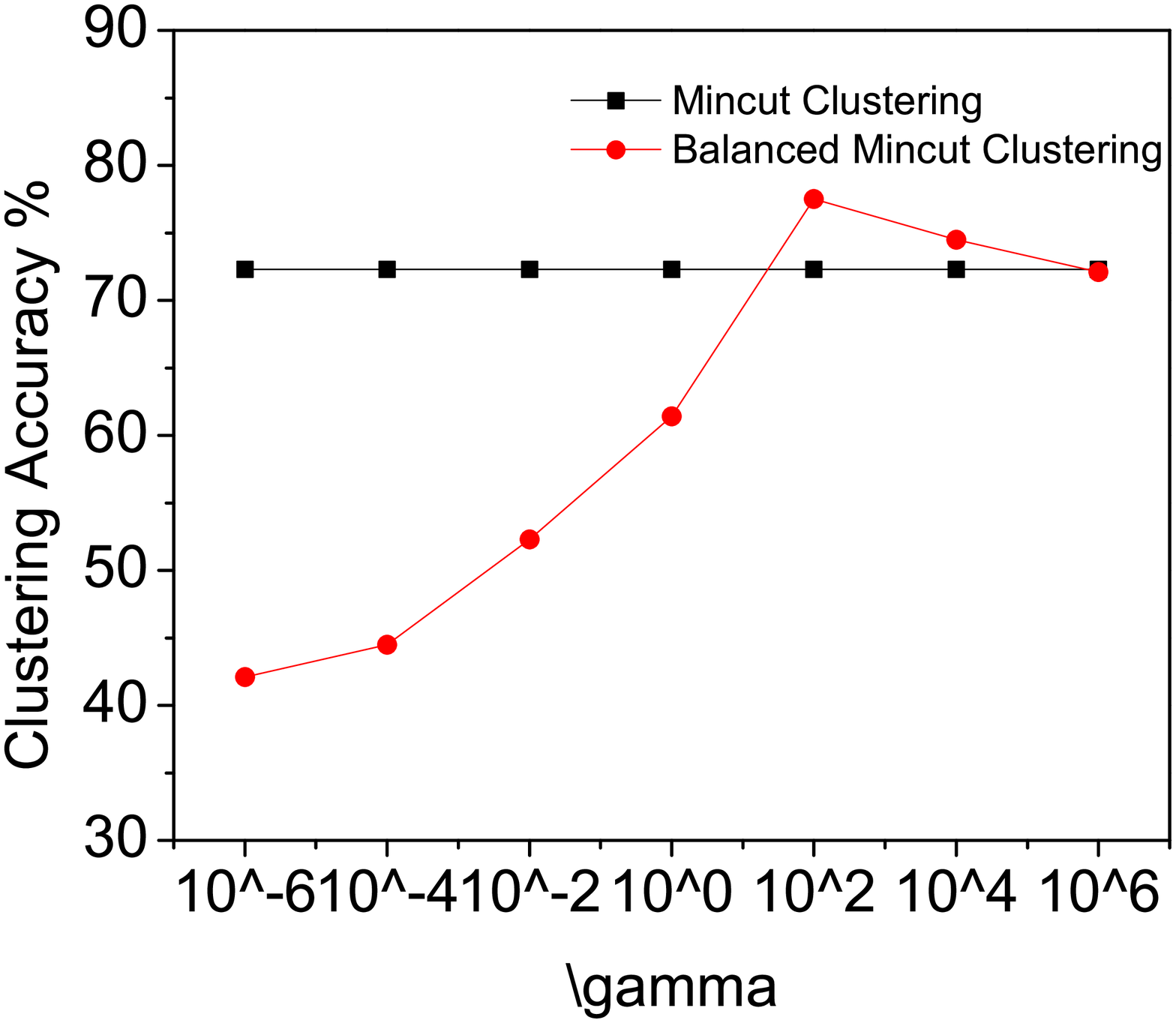}}
\subfigure[]{
\includegraphics[scale=0.21]{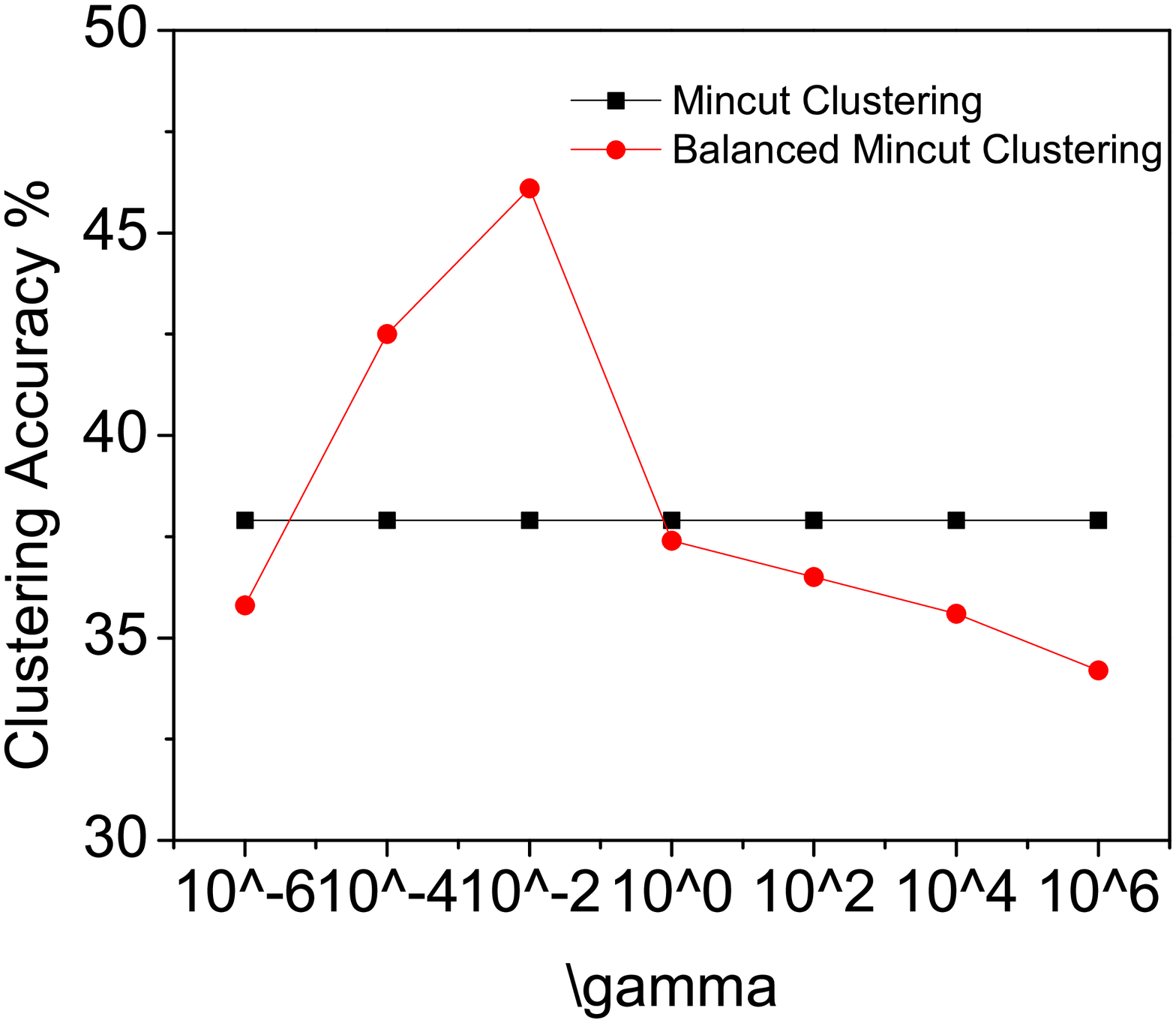}}
\subfigure[]{
\includegraphics[scale=0.21]{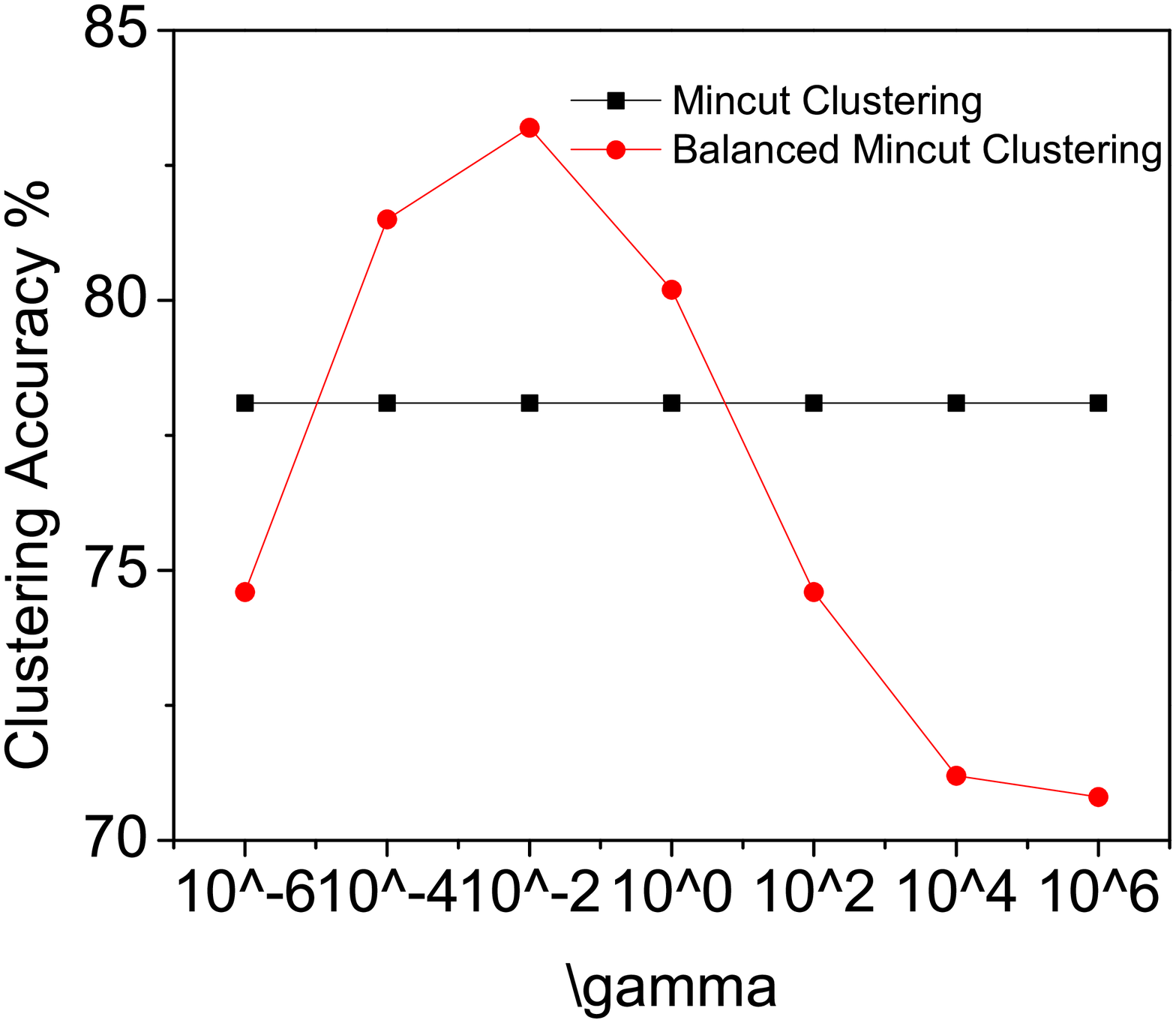}}
\subfigure[]{
\includegraphics[scale=0.21]{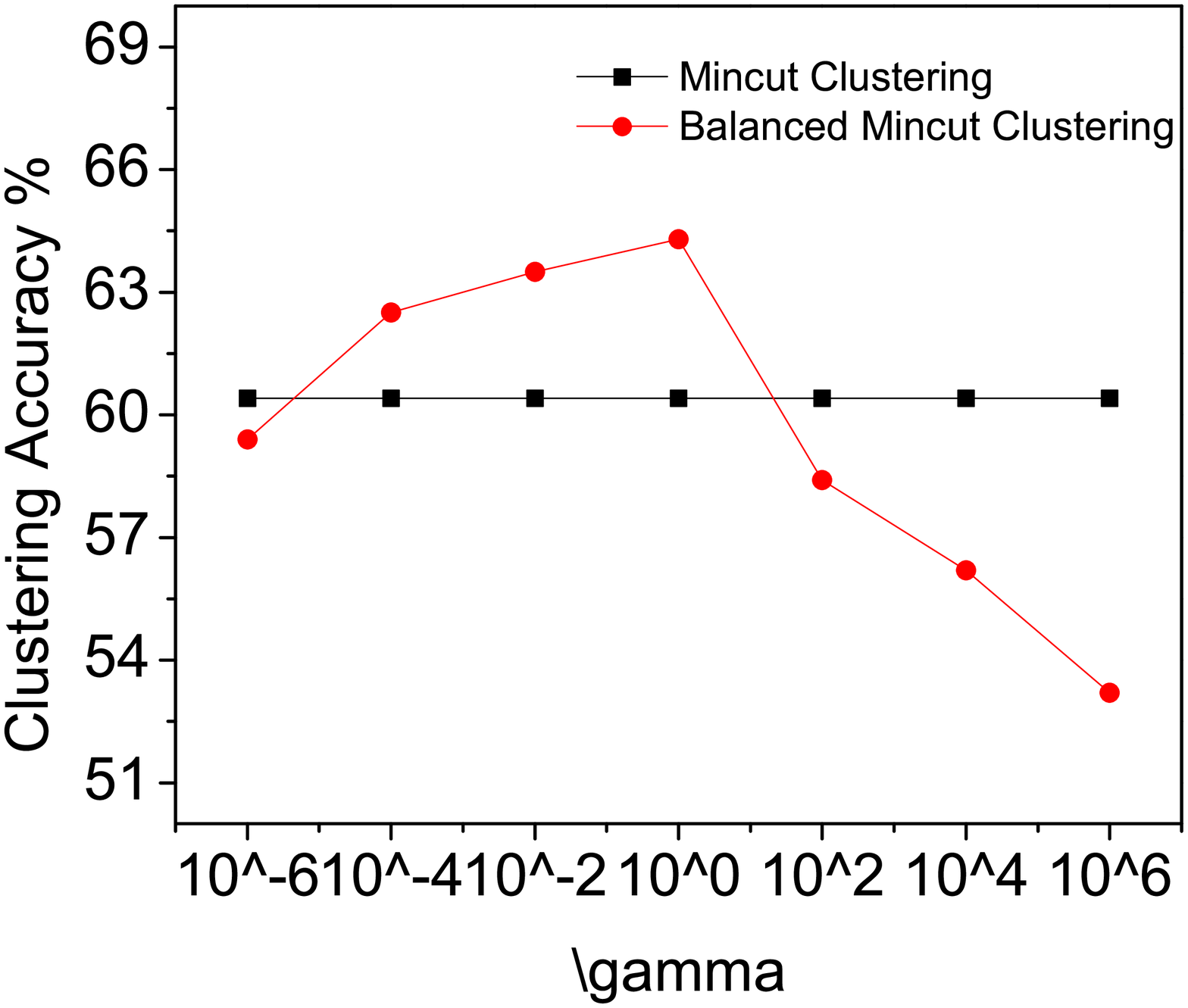}}
\subfigure[]{
\includegraphics[scale=0.21]{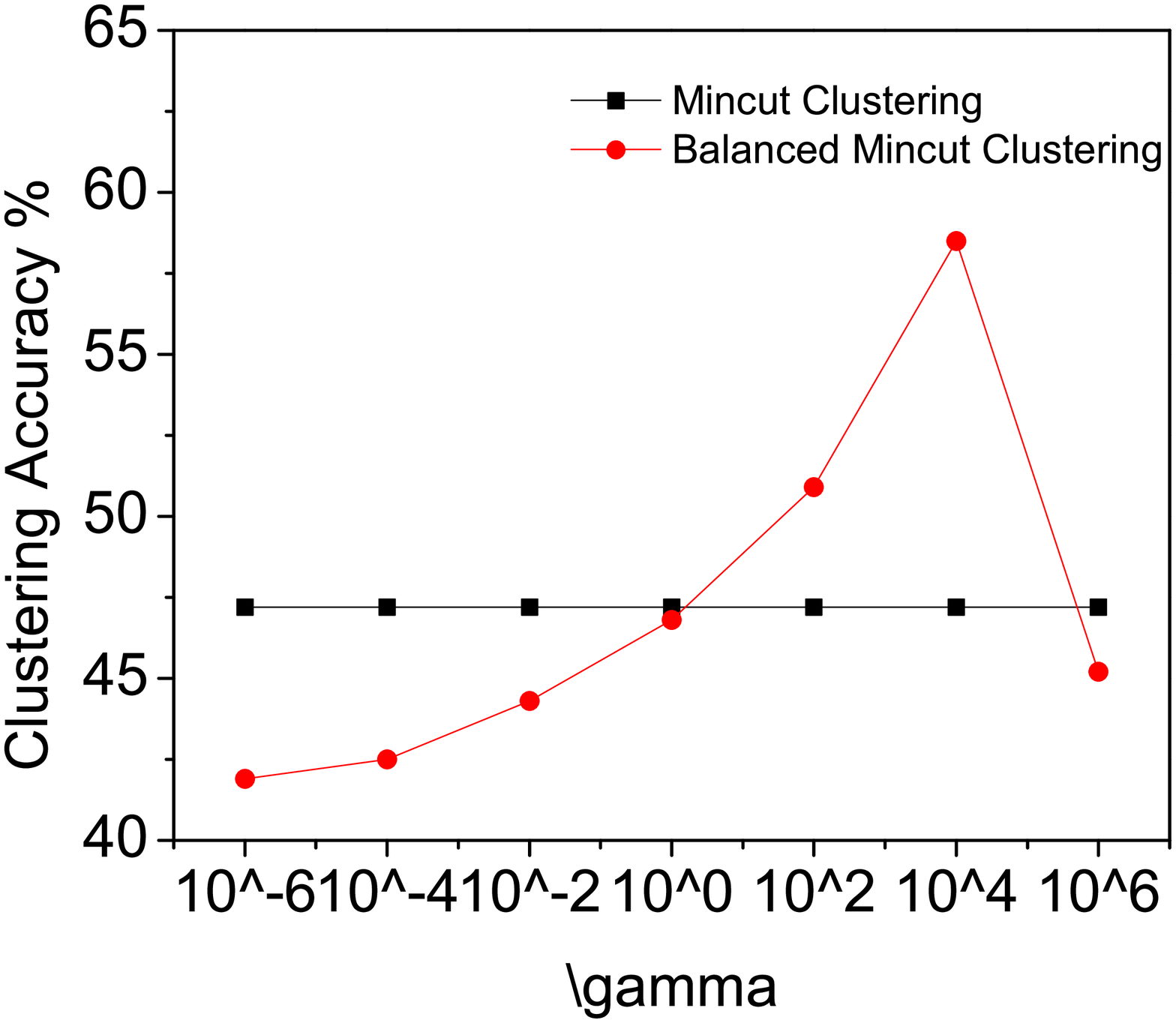}}
\subfigure[]{
\includegraphics[scale=0.21]{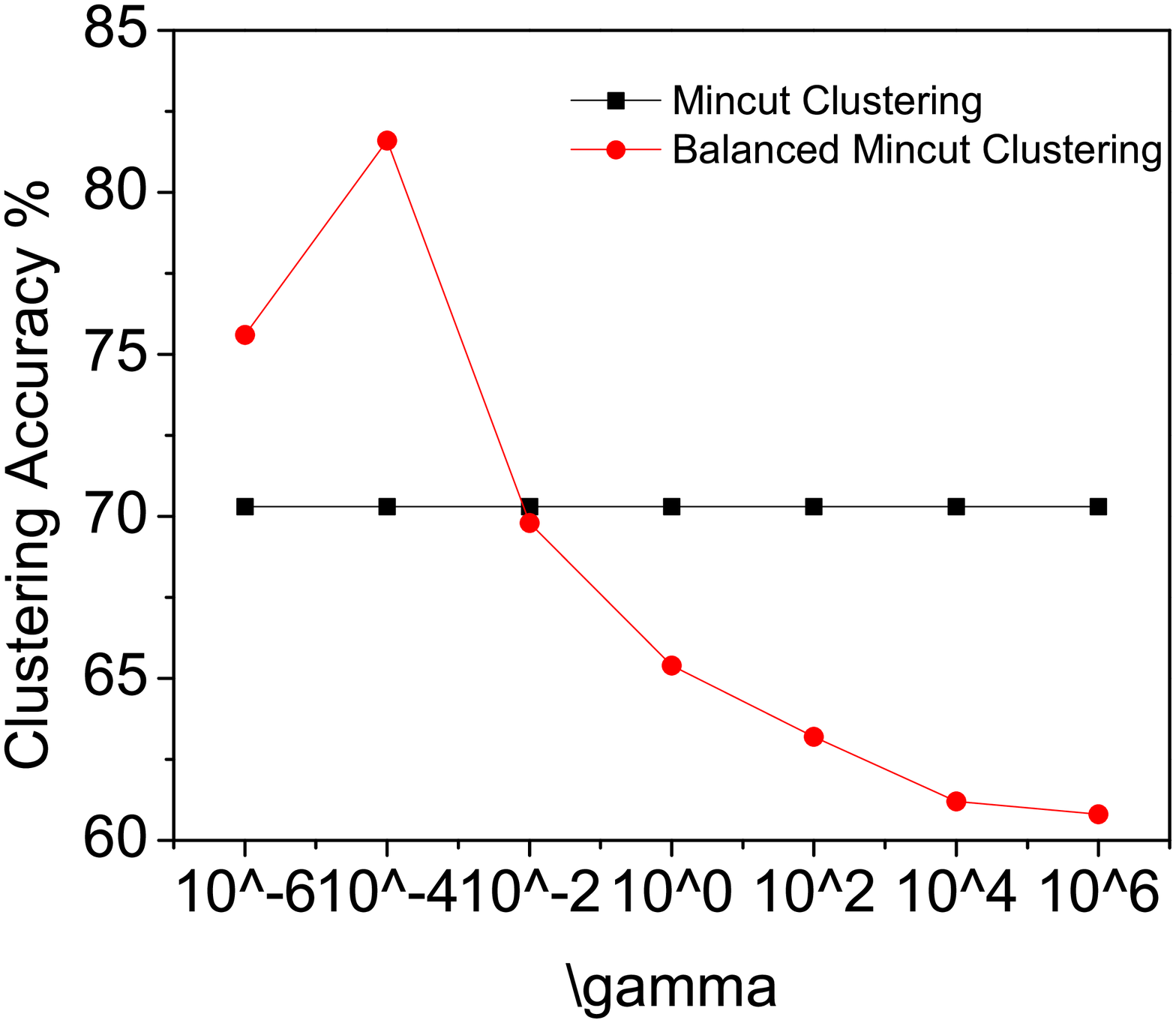}}
\subfigure[]{
\includegraphics[scale=0.21]{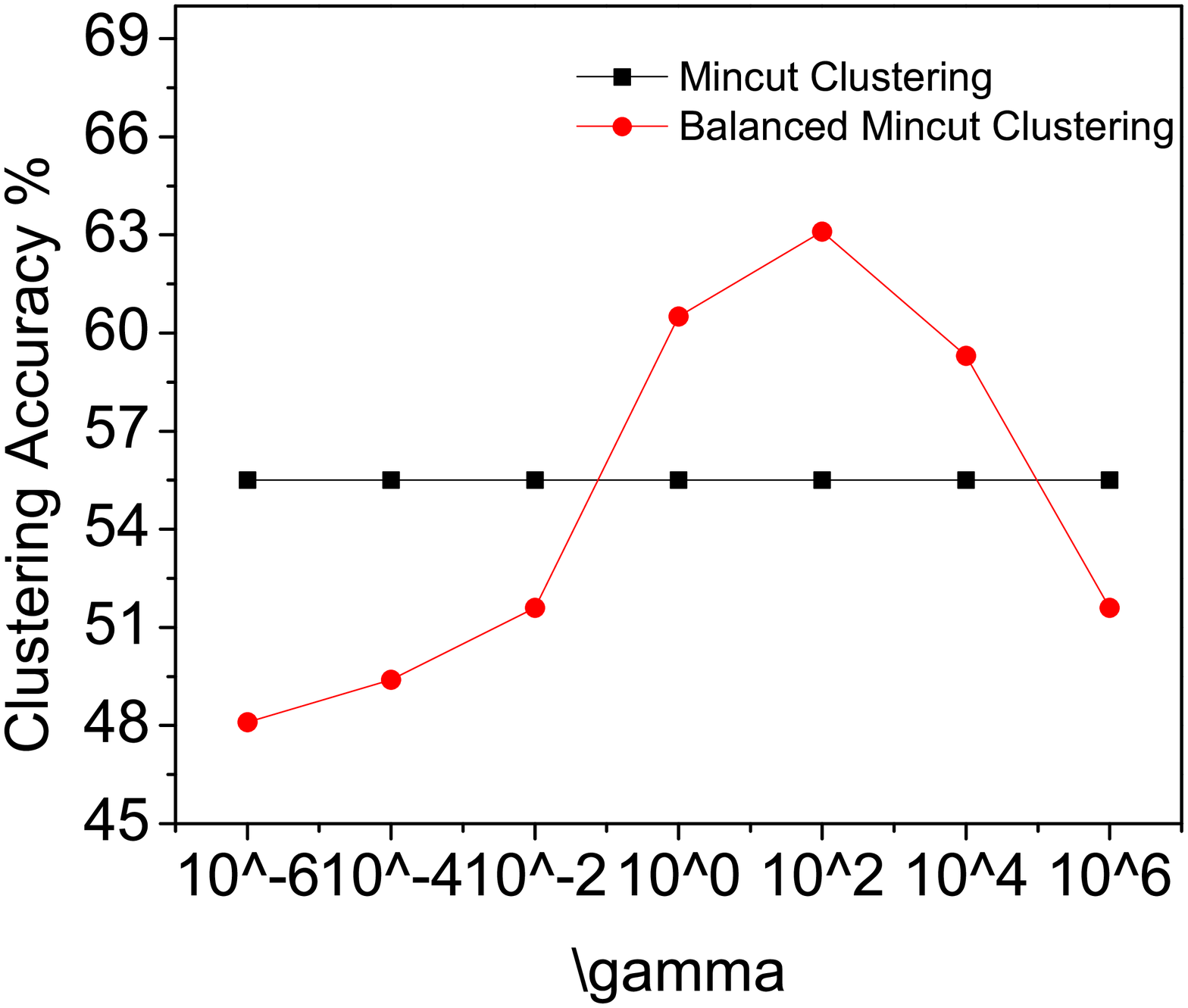}}
\subfigure[]{
\includegraphics[scale=0.21]{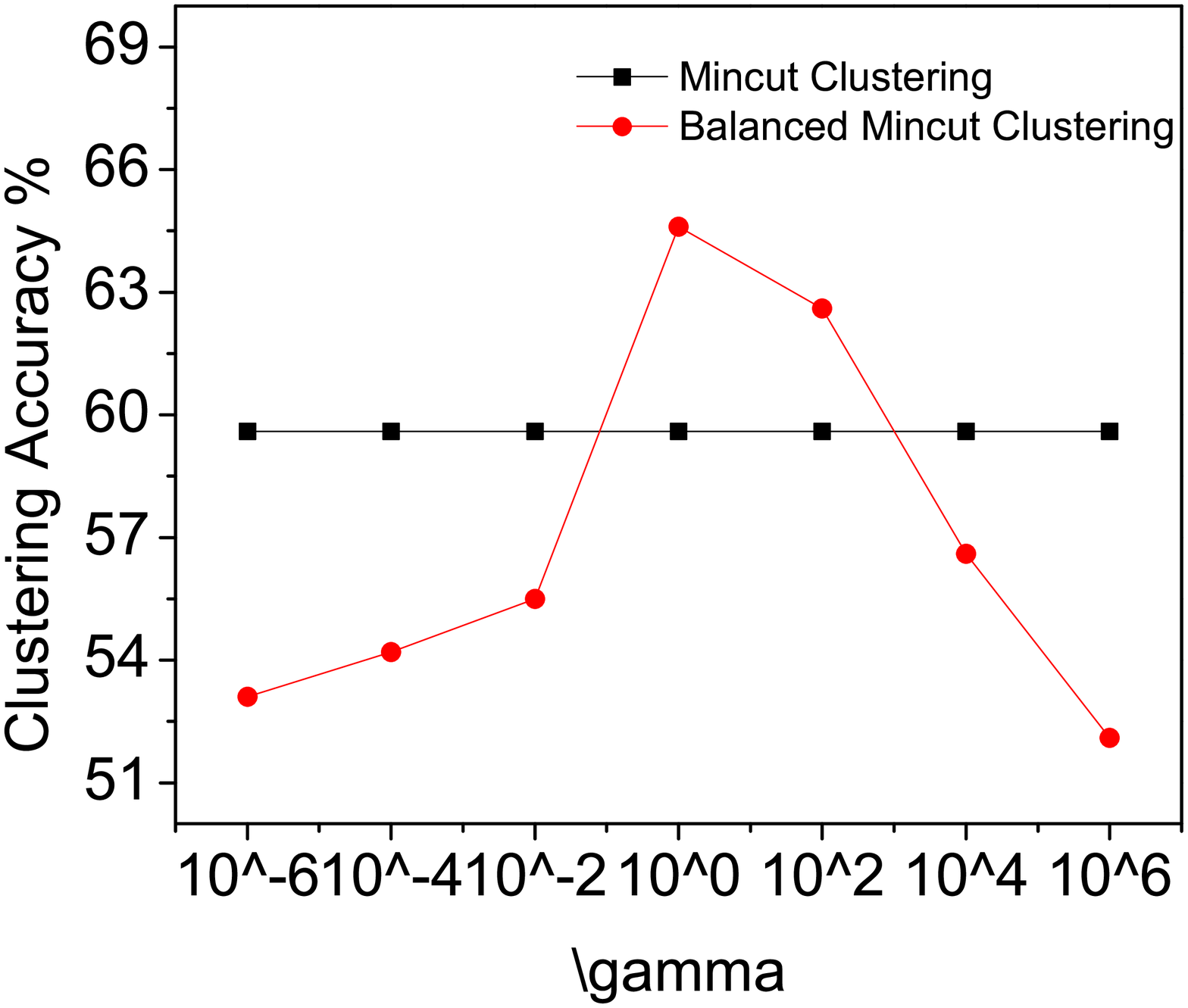}}
\caption{Parameter sensitivity of Balanced Min-Cut w.r.t $\gamma$. (a) MNIST (b) USPS (c) YaleB (d) ORL (e) JAFFE (f) HumanEVA (g) Coil20 (h) CMU-PIE (i) UMIST. From the results, we can observe that the parameter, $\gamma$ has a significant impact on the performance. To be more specific, better performance is achieved when $\gamma$ is in the range of $\{10^{-2}, 10^2\}$.}
\label{mincutsen}
\end{figure*}

\subsection{Comparison among graph clustering algorithms}

To evaluate performance of the proposed balanced min-cut clustering algorithm, we compare it to the classical Min-Cut clustering, MinMax Cut clustering \cite{minmax}, Ratio Cut clustering \cite{ratiocut}, Normalized Cut Clustering \cite{nomalizedcut} and Balanced Min-Cut clustering on the nine benchmark datasets.

We have the following observations from the experimental results:

\begin{enumerate}
\item Compared with the $k$-means based clustering, the graph clustering algorithms generally achieve better performance. This observation indicates that it is beneficial to utilize the pairwise similarities between all data points from a weighted graph adjacency matrix that contains much helpful information for clustering.

\item MinMax Cut Clustering always gets the second best performance, which demonstrates that min-max clustering principle can result in more balanced partitions than the other comparison graph clustering methods.

\item The proposed balanced min-cut clustering algorithm consistently outperforms the other graph clustering algorithms. From this result, we can conclude that the exclusive lasso is able to exert balance constraint on min-cut clustering and thus achieves the most balanced clustering result.
\end{enumerate}

\subsection{Parameter Sensitivity of the Proposed Algorithm}

In this section, we study the parameter sensitivity of balanced $k$-means and balanced min-cut. Fig \ref{kmeanssen} shows the accuracy ($y$-axis) of balanced $k$-means for different $\gamma$ values ($x$-axis). From the experimental result, we can observe that $\gamma$ has a significant impact on the performance of balanced $K$-means. 

We additionally show the parameter sensitivity of balanced min-cut in Fig. \ref{mincutsen}. Similarly to the proposed balanced $k$-means, the performance is heavily influenced by the parameter $\gamma$. To be more specific, better performance is usually attained when $\gamma$ is in the range of $\{10^{-2}, 10^2\}$. 

The experiments on both algorithms suggest the importance of designing an auto-tuning method for parameter selection. However, how to decide the optimal parameter is currently out of the scope in this work. We shall focus on this problem in the future.

\section{Conclusion}
In this paper, we have addressed the issue of balanced clustering which has not been studied in data mining. The exclusive lasso has been exploited to exert the balance constraint for introduce its ability to induce competition among different categories for the same data point. Particularly, we incorporated the exclusive lasso into $k$-means and min-cut clustering algorithms, which shall facilitate these two mainstream clustering algorithms to better cope with balanced data points. On the other hand, our objective functions are non-smooth and difficult to optimize. A new iterative approach is then designed to solve the problems. We have performed extensive experiments on a copious of datasets to evaluate performance of the proposed balanced $k$-means and balanced min-cut in terms of clustering accuracy and NMI. The experimental results show that our proposed algorithms always outperform the other comparison state-of-art clustering algorithms, which validates that utilizing the exclusive lasso indeed helps achieve the most balanced clustering.

\bibliographystyle{IEEEtran}

\bibliography{ICDE}

\end{document}